\documentclass[runningheads]{llncs}

\usepackage{eccv}

\usepackage{eccvabbrv}

\usepackage{graphicx}
\usepackage{booktabs}

\usepackage{amsmath}
\usepackage{amssymb}
\usepackage{mathtools}
\usepackage{algorithm}

\usepackage{algorithmic}

\usepackage{xcolor}

\usepackage{subcaption}

\usepackage{colortbl}

\usepackage{graphicx}

\usepackage{caption}

\usepackage{adjustbox}
\usepackage{multirow} 
\usepackage{makecell}
\definecolor{cvprblue}{rgb}{0.21,0.49,0.74}

\definecolor{CodeBG}{RGB}{255,255,255}      % white background
\definecolor{CodeFrame}{RGB}{225,225,235}   % light gray-lilac frame
\definecolor{FuncPurple}{RGB}{160,120,210}  % light purple for functions
\definecolor{CommentGreen}{RGB}{120,170,120}% light green for comments
\definecolor{ModelBlue}{RGB}{80,130,200}% light green for comments

\newcommand{\fn}[1]{\textcolor{FuncPurple}{#1}}     % function names
\newcommand{\cm}[1]{\textcolor{CommentGreen}{#1}}  % comments
\newcommand{\model}[1]{\textcolor{ModelBlue}{#1}}  % model

\usepackage{float}

\usepackage[accsupp]{axessibility}  % Improves PDF readability for those with disabilities.

\usepackage{hyperref}

\usepackage{orcidlink}

\begin{document}

\title{Transition Flow Matching} 

\titlerunning{Transition Flow Matching}

% \author{Chenrui Ma\inst{1}\orcidlink{0000-1111-2222-3333} }
\author{Chenrui Ma}

\authorrunning{Chenrui Ma}

% \institute{University of California, Irvine, Irvine, CA 92697, USA
% \email{chenrum@uci.edu}\\
% \url{https://merry7cherry.github.io/} }

\institute{University of California, Irvine, Irvine, CA 92697, USA \\
\email{chenrum@uci.edu} }

\maketitle

\begin{abstract}

Mainstream flow matching methods typically focus on learning the local velocity field, which inherently requires multiple integration steps during generation. In contrast, Mean Velocity Flow models establish a relationship between the local velocity field and the global mean velocity, enabling the latter to be learned through a mathematically grounded formulation and allowing generation to be transferred to arbitrary future time points.
In this work, we propose a new paradigm that directly learns the transition flow. As a global quantity, the transition flow naturally supports generation in a single step or at arbitrary time points. Furthermore, we demonstrate the connection between our approach and Mean Velocity Flow, establishing a unified theoretical perspective. Extensive experiments validate the effectiveness of our method and support our theoretical claims.
  
\end{abstract}

\section{Introduction} \label{sec:intro}

The goal of generative modeling is to transform a prior distribution into the data distribution. Flow Matching \cite{ma2024sit, karras2022elucidating, ma2025stochastic} provides an intuitive and conceptually simple framework for constructing flow paths that transport one distribution to another. Closely related to diffusion models \cite{albergo2023stochastic, cai2025diffusion, song2021scorebased}, Flow Matching focuses on learning the velocity fields that guide the generative process during training. 
% Since its introduction, Flow Matching has been widely adopted in modern generative modeling [11, 33, 35].

Both Flow Matching~\cite{lipman2023flow} and diffusion~\cite{song2021scorebased} models rely on iterative sampling during generation. Recently, significant attention has been devoted to few-step—and particularly one-step, feedforward—generative models. One common approach to accelerate Flow/Diffusion models is distillation. However, an ideal solution would allow models to be trained from scratch in an end-to-end manner without relying on pre-trained teachers. Pioneering this direction, Consistency Models \cite{song2023consistency, geng2025consistency, frans2025one} introduce a consistency constraint on network outputs for inputs sampled along the same trajectory. Despite their promising performance, this constraint is imposed as a behavioral property of the network, while the properties of the underlying ground-truth field that should guide learning remain unclear \cite{kim2024consistency, lu2025simplifying}. As a result, training can be unstable and typically requires a carefully designed discretization curriculum to progressively constrain the time domain \cite{song2024improved, song2023consistency, ma2025beyond}.
In contrast, Mean Velocity Models \cite{geng2025mean, zhang2025alphaflow, hu2025cmt, geng2025improved, lu2026one} optimize an objective derived from the relationship between the instantaneous velocity at each time point and the mean velocity toward a future time. This formulation provides a more fundamental and elegant perspective for learning generative dynamics.

In this work, we propose a principled and effective framework, termed \textbf{Transition Flow Matching}, for few-step generation with arbitrary numbers of steps and step sizes. Instead of regressing a local vector field as in Flow/Diffusion models \cite{ma2024sit, song2021scorebased, ma2025learningstraightflowsvariational}, our method directly models the generation trajectory itself, where the transition dynamics naturally represent a global counterpart of velocity. To this end, we derive (with proof) the \textbf{Transition Flow Identity}, and propose a principled training objective that enables generative models to satisfy this identity from scratch in an end-to-end manner. This formulation extends and generalizes previous Transition Models~\cite{wang2025transition, nie2026transition, luo2025soflow} and Flow Map methods~\cite{boffi2024flow, sabour2025align}. Importantly, we further establish a unified perspective that clarifies the relationship between our framework and Mean Velocity methods \cite{geng2025mean, geng2025improved}.
% Our method enables generation with arbitrary step sizes and arbitrary numbers of steps. 
In experiments, we conduct extensive evaluations, including generation trajectory visualization and image generation tasks across multiple datasets. The results show that our approach achieves competitive performance, demonstrating the effectiveness of modeling transition flows. In addition, we perform ablation studies to analyze key implementation design choices.

Our contributions can be summarized as follows:
• We propose Transition Flow Matching, a principled framework for few-step generative modeling.
• We derive the Transition Flow Identity and provide a theoretically grounded training objective.
• We establish a unified view connecting our method with Mean Velocity models.
• We demonstrate competitive performance across multiple image generation benchmarks.
\section{Related Works} \label{sec:rw}

\paragraph{Flow Matching and Diffusion.}
Flow Matching~\cite{lipman2023flow, albergo2023building, esser2024scaling} and Diffusion~\cite{albergo2023stochastic, cai2025diffusion, song2021scorebased} models generate samples by parameterizing continuous-time dynamics via Ordinary Differential Equations (ODEs) or Stochastic Differential Equations (SDEs). Flow Matching learns a deterministic velocity field defining an ODE over sample evolution~\cite{lipman2024flow}, while Diffusion models are typically formulated as SDEs that learn the score function~\cite{lai2025principlesdiffusionmodels}. Through the probability flow ODE formulation, the learned score implicitly defines an equivalent velocity field~\cite{lipman2024flowmatchingguidecode}.
In both cases, the learned velocity or score is a local quantity depending only on the current state and time. As a result, generation requires numerical integration over multiple steps to gradually transport samples from noise to data.
A common strategy to alleviate trajectory conflicts is distillation~\cite{liu2023flow, zhang2025towards, lee2024improving, esser2024scaling, geng2023onestep, salimans2022progressive, wang2024rectified} from a well-trained Flow or Diffusion model, which provides a better coupling distribution than the standard independent coupling setting~\cite{silvestri2025vct, tong2023improving, klein2023equivariant, anonymous2025cadvae}. In contrast, our Transition Flow Matching framework directly models global state transitions. Instead of regressing a local vector field, we parameterize mappings between states at different times, enabling direct transitions from a given state to an arbitrary future state without relying on local velocity regression.

\paragraph{Consistency Model and Mean Velocity Model.}
To reduce the number of inference steps, Consistency Models ~\cite{song2023consistency, geng2025consistency, kim2024consistency, lu2025simplifying, song2024improved, frans2025one} learn generation trajectories that maintain consistency over time, enforcing trajectory coherence. Meanwhile, Mean Velocity methods~\cite{geng2025mean, zhang2025alphaflow, hu2025cmt, geng2025improved, lu2026one} establish a relationship between the instantaneous velocity at each time point and a mean velocity target toward a future time.
By contrast, our method directly models the generation trajectory itself, where the transfer dynamics naturally represent a global counterpart of velocity. This enables generation with arbitrary step sizes and arbitrary numbers of steps. Importantly, we establish a unified perspective that clarifies the relationship between our framework and Mean Velocity methods~\cite{geng2025mean, geng2025improved}.

\paragraph{Transition Models and Flow Map.}
Transition-based approaches aim to learn state transitions by preserving transition identities. Consistency Trajectory Models~\cite{kim2024consistencytrajectorymodelslearning} learn mappings between arbitrary time steps, but rely on explicit ODE/SDE integration during training. Flow Map Matching~\cite{boffi2024flow} regresses zeroth- and first-order derivatives of flow fields, while IMM~\cite{zhou2025inductive} performs moment matching across time steps.
However, these methods either lack sufficient theoretical analysis, do not clearly connect to prior formulations, or rely on distillation. In contrast, our method is trained from scratch in an end-to-end manner and provides a clear theoretical relationship under both the Flow Matching~\cite{lipman2023flow} and Mean Velocity frameworks~\cite{geng2025mean, geng2025improved}.

% =========================
% Preliminary
% =========================
\section{Preliminaries}
\label{sec:prelim}

\subsection{Notation and setup}
\paragraph{Random variables and realizations.}
We work in $\mathbb{R}^d$. Uppercase letters (e.g., $X_t, Z$) denote random variables (RVs), and lowercase letters (e.g., $x_t, z$) denote their realizations. For the density of an RV $X_t$, we write $p_t(\cdot)$, and for a conditional density we write $p_{t\mid Z}(\cdot\mid z)$. Expectations are denoted by $\mathbb{E}[\cdot]$.

\paragraph{Source/target distributions and coupling.}
Let $X_0 \sim p_0$ be the \emph{source} distribution (e.g., standard Gaussian noise) and $X_1 \sim p_1$ be the \emph{target} distribution (e.g., images).
A generative model constructs a continuous path of distributions $\{p_t\}_{t\in[0,1]}$ that transports $p_0$ to $p_1$.
Let $(X_0,X_1)$ be any coupling with joint density $\pi$ whose marginals are $p_0$ and $p_1$.
Throughout this work, we follow the standard Flow Matching setting where the source and target are independent:
$\pi(x_0,x_1)=p_0(x_0)p_1(x_1)$.

\paragraph{Conditioning variable and conditional paths.}
We use a conditioning RV $Z$ to index conditional paths.
Conditioned on $Z=z$, we obtain a conditional path $\{X_t^Z\}_{t\in[0,1]}$ with conditional density $p_{t\mid Z}(\cdot\mid z)$.
The corresponding marginal path $\{X_t\}_{t\in[0,1]}$ has density $p_t(\cdot)$ satisfying
\begin{equation}
\label{eq:bayes_xt}
\begin{aligned}
p_t(x_t)
\;=\;
\int p_{t\mid Z}(x_t\mid z)\,p_Z(z)\,dz, \qquad
X_t\sim p_t,\ X_t\mid(Z=z)\sim p_{t\mid Z}(\cdot\mid z)
\end{aligned}
\end{equation}

\paragraph{General interpolant.}
We consider a general interpolant between the \emph{marginal} endpoints $X_0$ and $X_1$ specified by scalar functions
$\alpha:[0,1]\to\mathbb{R}$ and $\beta:[0,1]\to\mathbb{R}$:
\begin{equation}
\label{eq:generalInterpolant}
X_t \;=\; \alpha(t)\,X_0 + \beta(t)\,X_1,
\qquad t\in[0,1].
\end{equation}
We assume boundary conditions $\alpha(0)=1,\ \beta(0)=0$ and $\alpha(1)=0,\ \beta(1)=1$, so that $X_{t=0}=X_0$ and $X_{t=1}=X_1$.
If $\alpha,\beta$ are differentiable, then
\begin{equation}
\label{eq:generalInterpolant_dot}
\frac{d}{dt}X_t
\;=\;
\dot{\alpha}(t)\,X_0 + \dot{\beta}(t)\,X_1.
\end{equation}
Conditioned on $Z=z$, the same schedules induce the conditional path
$X_t^Z=\alpha(t)X_0^Z+\beta(t)X_1^Z$ and
$\frac{d}{dt}X_t^Z=\dot{\alpha}(t)X_0^Z+\dot{\beta}(t)X_1^Z$.

\subsection{Flow Matching}
\paragraph{Velocity fields.}
Let $v(x_t,t\mid z)\in\mathbb{R}^d$ denote a \emph{conditional} velocity field that transports $p_{t\mid Z}(\cdot\mid z)$ over time.
The corresponding \emph{marginal} velocity is defined by conditional expectation:
\begin{equation}
\label{eq:marginalV}
\begin{aligned}
v(x_t,t)
\;=\;
\int v(x_t,t\mid z)\,p_{Z\mid t}(z\mid x_t)\,dz 
\;=\; 
\mathbb{E}\!\left[\,v(X_t,t\mid Z)\,\middle|\,X_t=x_t\,\right]
\end{aligned}
\end{equation}

\paragraph{Generation ODE and continuity equation.}
A marginal trajectory follows the ODE
\begin{equation}
\label{eq:ode_marginal}
\frac{d x_t}{dt} = v(x_t,t), \qquad t\in[0,1],
\end{equation}
and similarly a conditional trajectory follows $\frac{d x_t^z}{dt}=v(x_t,t\mid z)$.
The induced marginal density $p_t$ satisfies the continuity (transport) equation
\begin{equation}
\label{eq:continuity_marginal}
\partial_t p_t(x_t) + \nabla\!\cdot\!\big(p_t(x_t)\,v(x_t,t)\big)=0,
\qquad t\in[0,1]
\end{equation}
and conditioned on $Z=z$, $p_{t\mid Z}(\cdot\mid z)$ satisfies
\begin{equation}
\label{eq:continuity_conditional}
\partial_t p_{t\mid Z}(x_t\mid z) + \nabla\!\cdot\!\big(p_{t\mid Z}(x_t\mid z)\,v(x_t,t\mid z)\big)=0,
\ \ t\in[0,1]
\end{equation}

\paragraph{Learning Flow Matching.}
Flow Matching introduces a parameterized model $v^\theta(X_t,t)$ to learn $v(X_t,t)$ by minimizing the marginal loss
\begin{equation}
\label{eq:MFM}
\mathcal{L}_\mathrm{MFM}(\theta) = \mathbb{E}_{t,\;X_t\sim p_t}\,
D\Big(v(X_t,t),\,v^\theta(X_t,t)\Big),
\end{equation}
where $D(\cdot,\cdot)$ is typically a Bregman divergence (e.g., MSE).
Since the marginal velocity $v(X_t,t)$ in Eq.~\eqref{eq:marginalV} is generally intractable, one instead minimizes the conditional loss
\begin{equation}
\label{eq:CFM}
\mathcal{L}_\mathrm{CFM}(\theta) = \mathbb{E}_{t,\;Z,\;X_t\sim p_{t\mid Z}(\cdot \mid Z)}\,
D\Big(v(X_t,t \mid Z),\,v^\theta(X_t,t)\Big)
\end{equation}

\begin{theorem}[Gradient equivalence of Flow Matching~\cite{lipman2024flowmatchingguidecode}]
\label{thm:MequivC_FM}
The gradients of the marginal Flow Matching loss and the conditional Flow Matching loss coincide:
\begin{equation}
\label{eq:marginalEqvconditional_FM}
\nabla_\theta \mathcal{L}_\mathrm{MFM}(\theta)
\;=\;
\nabla_\theta \mathcal{L}_\mathrm{CFM}(\theta).
\end{equation}
In particular, the minimizer of the conditional Flow Matching loss is the marginal velocity $v(x_t,t)$.
\end{theorem}

\begin{remark}[Standard Flow Matching]
\label{remark:FM}
Flow Matching sets the conditioning variable to be the endpoint pair $Z=(X_0,X_1)$.
Choosing linear schedules $\alpha(t)=1-t$ and $\beta(t)=t$ yields the conditional path
$
X_t^Z=(1-t)X_0+tX_1
$
with constant conditional velocity
$
v(X_t,t\mid Z)=X_1-X_0.
$
Thus Eq.~\eqref{eq:CFM} reduces to
\begin{equation}
\begin{aligned}
\mathcal{L}_\mathrm{FM}(\theta)
=
\mathbb{E}_{t,\;X_0\sim p_0,\;X_1\sim p_1}
D\Big(X_1-X_0,\;v^\theta(X_t,t)\Big),
X_t=(1-t)X_0+tX_1
\end{aligned}
\end{equation}
\end{remark}

% =========================
% Method
% =========================
\section{Method}
\label{sec:method}

% \subsection{Overview: learning a transition operator instead of a local velocity}
Flow Matching learns a time-dependent velocity field $v(x_t,t)$ and generating by integrating the ODE in Eq.~\eqref{eq:ode_marginal}.
Our goal is to directly learn a \emph{transition flow}:
\[
X^\theta(x_t,t,r)\;:\;\mathbb{R}^d\times[0,1]\times[0,1]\to\mathbb{R}^d,\quad 0\le t\le r\le 1,
\]
that maps a state $x_t$ at time $t$ to the \emph{future} state at time $r$ along the same transport dynamics.
At inference time, this enables stepping across an arbitrary time grid by repeatedly applying $X^\theta(\cdot)$, without explicitly integrating $v(\cdot)$.

\subsection{Transition Flow Identity}
\paragraph{Average velocity between two time steps.}
Given the marginal velocity $v(x_\tau,\tau)$ as shown in Eq.\eqref{eq:marginalV}, define the average velocity $u(x_t,t,r)$ for any $0\le t\le r\le 1$:
\begin{equation}
\label{eq:averageV}
(r-t)u(x_t,t,r)
=
x_{t\to r}-x_t
=
\int_t^r v(x_\tau,\tau)\,d\tau .
\end{equation}
Here $x_t$ is the current state at time $t$, and $x_{t\to r}$ denotes the marginal transition state at time $r$ obtained by evolving from $x_t$.

\paragraph{Marginal/conditional transition states.}
Analogous to Eq.~\eqref{eq:marginalV}, the marginal transition state can be expressed as a conditional expectation:
\begin{equation}
\label{eq:marginalstate}
x_{t \to r}
=
\int x_{t \to r}^z \; p_{Z\mid t}(z\mid x_t)\,dz
=
\mathbb{E}\!\left[\,X_{t \to r}^Z \,\middle|\, X_t=x_t\,\right],
\end{equation}
where the conditional transition state is
$
x_{t \to r}^z
=
x_t+\int_t^r v(x_\tau,\tau\mid z)\,d\tau.
$

\paragraph{Transition flow.}
We define the (marginal) transition flow $X(x_t,t,r)$ as the mapping that returns the marginal transition state:
\begin{equation}
\label{eq:marginalTransitionFlow}
\begin{aligned}
X(x_t,t,r)
=
\int X(x_t,t,r\mid z)\;p_{Z\mid t}(z\mid x_t)\,dz 
=
\mathbb{E}\!\left[\,X(X_t,t,r\mid Z)\,\middle|\,X_t=x_t\,\right]
\end{aligned}
\end{equation}
where $X(x_t,t,r\mid z)=x_{t\to r}^z$.
Combining Eq.~\eqref{eq:averageV} with the definition of $X(\cdot)$ yields
\begin{equation}
\label{eq:u2x}
u(x_t,t,r)
=
\frac{x_{t\to r}-x_t}{r-t}
=
\frac{X(x_t,t,r)-x_t}{r-t}.
\end{equation}

\paragraph{Transition Flow Identity.}
Differentiating Eq.~\eqref{eq:averageV} with respect to $t$ (treating $r$ as independent of $t$) gives
\begin{equation}
\label{eq:u_identity}
u(x_t,t,r)
=
v(x_t,t)
+
(r-t)\frac{d}{dt}u(x_t,t,r).
\end{equation}
Substituting Eq.~\eqref{eq:u2x} into Eq.~\eqref{eq:u_identity} yields the following key identity:
\begin{equation}
\label{eq:transitionIdentity}
\boxed{
X(x_t,t,r) = x_{t \to r} + (r-t) \frac{d}{dt} X(x_t,t,r)
}
\end{equation}
We defer the detailed algebraic derivation to the Appendix.
% We defer the detailed algebraic derivation to Appendix~\S\ref{proof:sec:TranFlowIden}.

\subsection{Computing the time derivative of a transition flow}
To make Eq.~\eqref{eq:transitionIdentity} actionable for learning, we expand the total derivative $\frac{d}{dt}X(x_t,t,r)$ as:
\begin{equation}
\label{eq:X_timederivative}
\begin{aligned}
\frac{d}{dt} X(x_t,t,r)
&=
\partial_{x_t}X \cdot \frac{d x_t}{dt}
+
\partial_{t}X \cdot \frac{d t}{dt}
+
\partial_{r}X \cdot \frac{d r}{dt} \\
&=
\partial_{x_t}X \cdot v(x_t,t)
+
\partial_{t}X
\end{aligned}
\end{equation}
where $\frac{d x_t}{dt}=v(x_t,t)$ by Eq.~\eqref{eq:ode_marginal} and $\frac{dr}{dt}=0$.
In practice, Eq.~\eqref{eq:X_timederivative} is given by the Jacobian-vector product (JVP) between the Jacobian matrix of each function ($[\partial_{x_t}X, \partial_{t}X, \partial_{r}X]$) and the corresponding tangent vector ($[v, 0, 1]$). For code implementation, modern libraries such as PyTorch provide efficient JVP calculation interfaces.

\subsection{Transition Flow Matching objectives}
\paragraph{Model parameterization.}
We introduce a Transition Flow model $X^\theta(x_t,t,r)$ to model $X(x_t,t,r)$.
We use $\mathrm{sg}[\cdot]$ to denote a stop-gradient operator (i.e., the argument is treated as a constant target during optimization), and $D(\cdot,\cdot)$ is a Bregman divergence (e.g., MSE).

\paragraph{Intractable marginal objective.}
Ideally, we would enforce Eq.~\eqref{eq:transitionIdentity} by minimizing the marginal Transition Flow Matching objective (M-TFM):
\begin{equation}
\label{eq:MTFM}
\begin{aligned}
\mathcal{L}_\mathrm{M\text{-}TFM}(\theta)
\;=\;
\mathbb{E}_{t,\;r,\;X_t\sim p_t}\,
D\Big(\mathrm{sg}\big[X_{\mathrm{tgt}}^{\mathrm{m}}(X_t,t, r)\big],\,X^\theta(X_t,t, r)\Big),
\end{aligned}
\end{equation}
with target
\begin{equation}
\label{eq:marginalXtgt}
X_{\mathrm{tgt}}^{\mathrm{m}}(X_t,t, r)
=
X_{t \to r}
+
(r-t) \frac{d}{dt} X^\theta(X_t,t,r).
\end{equation}
However, the marginal transition state $X_{t\to r}$ and the marginal velocity $v(x_t,t)$ (needed inside $\frac{d}{dt}X^\theta$ via Eq.~\eqref{eq:X_timederivative}) are generally intractable, hence Eq.~\eqref{eq:MTFM} cannot be evaluated directly.

\paragraph{Tractable conditional objective.}
Instead, we minimize a conditional Transition Flow Matching objective (C-TFM):
\begin{equation}
\label{eq:CTFM}
\begin{aligned}
\mathcal{L}_\mathrm{C\text{-}TFM}(\theta)
\; = \;
\mathbb{E}_{t,\;r,\;Z,\;X_t\sim p_{t\mid Z}(\cdot \mid Z)}\, 
D\Big(\mathrm{sg}\big[X_{\mathrm{tgt}}^{\mathrm{c}}(X_t,t, r \mid Z)\big],\,X^\theta(X_t,t, r)\Big),
\end{aligned}
\end{equation}
where the conditional target is
\begin{equation}
\label{eq:conditionalXtgt}
X_{\mathrm{tgt}}^{\mathrm{c}}(X_t,t, r \mid Z)
=
X_{t \to r}^Z
+
(r-t) \frac{d}{dt} X^\theta(X_t,t,r).
\end{equation}
Crucially, under appropriate choices of $Z$ and the conditional path construction, both the conditional transition state $X_{t\to r}^Z$ and the conditional velocity $v(x_t,t\mid Z)$ become tractable, which makes $\frac{d}{dt}X^\theta$ computable via Eq.~\eqref{eq:X_timederivative}.

\begin{theorem}[Gradient equivalence of Transition Flow Matching]
\label{thm:MequivC}
% see the proof in Appendix~\S\ref{proof:sec:M2C}.
See the proof in the Appendix.
The gradients of the marginal and conditional Transition Flow Matching losses coincide:
\begin{equation}
\label{eq:marginalEqvconditional}
\nabla_\theta \mathcal{L}_\mathrm{M\text{-}TFM}(\theta)
\;=\;
\nabla_\theta \mathcal{L}_\mathrm{C\text{-}TFM}(\theta).
\end{equation}
In particular, the minimizer of the conditional objective regresses $X^\theta(X_t,t,r)$ toward the marginal target
$X_{\mathrm{tgt}}^{\mathrm{m}}(X_t,t,r)$, which is defined to satisfy the Transition Flow Identity in Eq.~\eqref{eq:transitionIdentity}.
\end{theorem}
% The full proof (including the role of Bregman divergences and the marginal--conditional relations) is provided in Appendix~\ref{proof:transitionIdentity}.

\subsection{Training and inference procedure} \label{sec:training_and_infer}
\begin{remark}[Standard Transition Flow Matching]
\label{remark:TFM}
We adopt the standard Flow Matching setting in Remark~\ref{remark:FM} by taking
$Z=(X_0,X_1)$ and using the linear interpolant
\[
X_t=(1-t)X_0+tX_1,\qquad 0\le t\le 1.
\]
For any $0\le t\le r\le 1$, the conditional transition state is
\[
X_{t\to r}^Z=(1-r)X_0+rX_1,
\]
and the conditional velocity is constant:
\[
v(X_t,t\mid Z)=X_1-X_0.
\]
Under this instantiation, the tractable Transition Flow Matching objective in Eq.~\eqref{eq:CTFM} becomes
\begin{equation}
\label{eq:STFM}
\begin{aligned}
\mathcal{L}_\mathrm{TFM}(\theta)
=
\mathbb{E}_{t,\;r,\;Z,\;X_t\sim p_{t\mid Z}(\cdot \mid Z)}\, 
D\Big(\mathrm{sg}\big[X_{t \to r}^Z + (r-t)\frac{d}{dt} X^\theta(X_t,t,r)\big],\,X^\theta(X_t,t, r)\Big),
\end{aligned}
\end{equation}
The time derivative term $\frac{d}{dt} X^\theta(X_t,t,r)$ extends Eq.~\eqref{eq:X_timederivative} by incorporating the tractable conditional velocity, and is given by (see the proof in Appendix):
% The time derivative term $\frac{d}{dt} X^\theta(X_t,t,r)$ extends Eq.~\eqref{eq:X_timederivative} by incorporating the tractable conditional velocity, and is given by (see the proof in Appendix~\S\ref{proof:sec:M2C}):
\begin{equation}
\label{eq:condTimeDX}
\begin{aligned}
\frac{d}{dt} X^\theta(X_t,t,r)
\;=\; 
\partial_{x_t}X^\theta(X_t,t,r)\cdot v(X_t, t \mid Z)
+
\partial_{t}X^\theta(X_t,t,r).
\end{aligned}
\end{equation}
\end{remark}

During inference, we recursively apply the model to transform the generation trajectory from the source distribution to the target distribution:
\begin{equation}
\hat{x_r} = X^\theta (x_t, t, r)
\end{equation}
For clarity, we summarize the conceptual training procedures in the Algorithm.\ref{alg:tfm-training-short}. 
% Note that in practice, we need some tricks to stabilize the training procedure, which we showcase in Appendix~\S\ref{}.

\begin{algorithm}[t]
\caption{Transition Flow Matching (TFM): Training}
\label{alg:tfm-training-short}

\setlength{\fboxsep}{8pt}
\noindent
\fcolorbox{CodeFrame}{CodeBG}{
\begin{minipage}{0.933\linewidth}
\ttfamily\small
\cm{// jvp returns (output, JVP)}\\
\cm{// tfm predicts $X^\theta(x_t,t,r)$}\\[2pt]
x\_1 = \fn{sample\_batch}(), \;
x\_0 = \fn{randn\_like}(x\_1), \;
(t,r)=\fn{sample\_t\_r}() \hfill \cm{// $0\le t\le r\le1$}\\[3pt]
x\_t = (1-t)x\_0 + t x\_1,\;
v = x\_1 - x\_0\\
x\_{t\_to\_r} = (1-r)x\_0 + r x\_1\\[3pt]
(X, dX\_dt) = \fn{jvp}(\model{tfm}, (x\_t,t,r),(v,1,0))\\
X\_tgt = x\_{t\_to\_r} + (r-t)dX\_dt\\
loss = \fn{metric}(X - \fn{stopgrad}(X\_tgt))\\
\textbf{return} loss
\end{minipage}
}
\end{algorithm}

% \subsection{Classifier-Free Guidance}
% Transition Flow Matching naturally supports classifier-free guidance (CFG) as:
% \begin{equation}
% \begin{aligned}
% X^\theta_\text{cfg}(x_t,t,r \mid \textbf{c}) = \omega X(x_t,t,r \mid \textbf{c}) + (1-\omega) X(x_t,t,r \mid \varnothing) 
% \end{aligned}
% \end{equation}
% which is the combination of class-conditional output $X(x_t,t,r \mid \textbf{c})$ and unconditional output $X(x_t,t,r \mid \varnothing)$, so that we can modify $\omega$ to adjust the class-conditional guidance in the inference time. Note in here the condition refers to class-condition, different from the meaning of condition (with marginal) in previous sections.
% Follow the common practice of CFG, we train $X^\theta(\cdot)$ that both support class-condition generation $X^\theta(\cdot \mid \textbf{c})$ and uncondition generation $X^\theta(\cdot \mid \varnothing)$. 
% For conditional training, the data point $X_1$ in Sec.~\ref{sec:training_and_infer} are sampled from the class-condition target distribution $p_1(\cdot \mid \textbf{c})$, in contrast, for unconditional training, the data point $X_1$ in Sec.~\ref{sec:training_and_infer} are sampled from the unconditional target distribution $p_1(\cdot)$,

\subsection{Classifier-Free Guidance}
Transition Flow Matching naturally supports classifier-free guidance (CFG) via
\begin{equation}
\begin{aligned}
X^\theta_\text{cfg}(x_t,t,r \mid \textbf{c}) = \omega X(x_t,t,r \mid \textbf{c}) + (1-\omega) X(x_t,t,r \mid \varnothing) .
\end{aligned}
\end{equation}
This forms a combination of the class-conditional transition flow $X(x_t,t,r \mid \textbf{c})$ and the unconditional transition flow $X(x_t,t,r \mid \varnothing)$, allowing us to control the strength of class conditioning at inference time by tuning $\omega$. Here, the term ``condition'' refers specifically to class conditioning, which is distinct from the marginal/conditional distinction used in previous sections.
Following standard CFG practice~\cite{ho2022classifierfreediffusionguidance}, we train a single model $X^\theta(\cdot)$ that supports both class-conditional generation $X^\theta(\cdot \mid \textbf{c})$ and unconditional generation $X^\theta(\cdot \mid \varnothing)$. Concretely, for conditional training, the endpoint $X_1$ in Sec.~\ref{sec:training_and_infer} is sampled from the class-conditional target distribution $p_1(\cdot \mid \textbf{c})$; for unconditional training, $X_1$ in Sec.~\ref{sec:training_and_infer} is sampled from the unconditional target distribution $p_1(\cdot)$.

% \begin{equation}
% \begin{aligned}
% X(X_t,t,r \mid \textbf{c}) &= X_{t\to r}^{Z,c} + (r-t)\frac{d}{dt} X^\theta(X_t,t,r \mid \textbf{c}) \\
% \frac{d}{dt} X^\theta(X_t,t,r \mid \textbf{c}) &=  \partial_{x_t}X^\theta(X_t,t,r \mid \textbf{c})\cdot v(X_t, t \mid Z,\textbf{c}) + \partial_{t}X^\theta(X_t,t,r \mid \textbf{c})
% \end{aligned}
% \end{equation}

% \begin{equation}
% \begin{aligned}
% X(X_t,t,r) = X_{t\to r}^Z + (r-t)\frac{d}{dt} X^\theta(X_t,t,r)
% \end{aligned}
% \end{equation}

% \section{Discussion}

\section{Implementation Design}
\label{sec:implementation}

% Required packages in preamble:
% \usepackage{graphicx}

\begin{figure*}[t]
  \centering

  % ---------------- Row 1: Block 1 and Block 2 ----------------
  % Block 1
  \begin{minipage}[t]{0.49\textwidth}
    \centering
    \begin{minipage}[t]{0.33\linewidth}
      \centering
      \includegraphics[width=\linewidth]{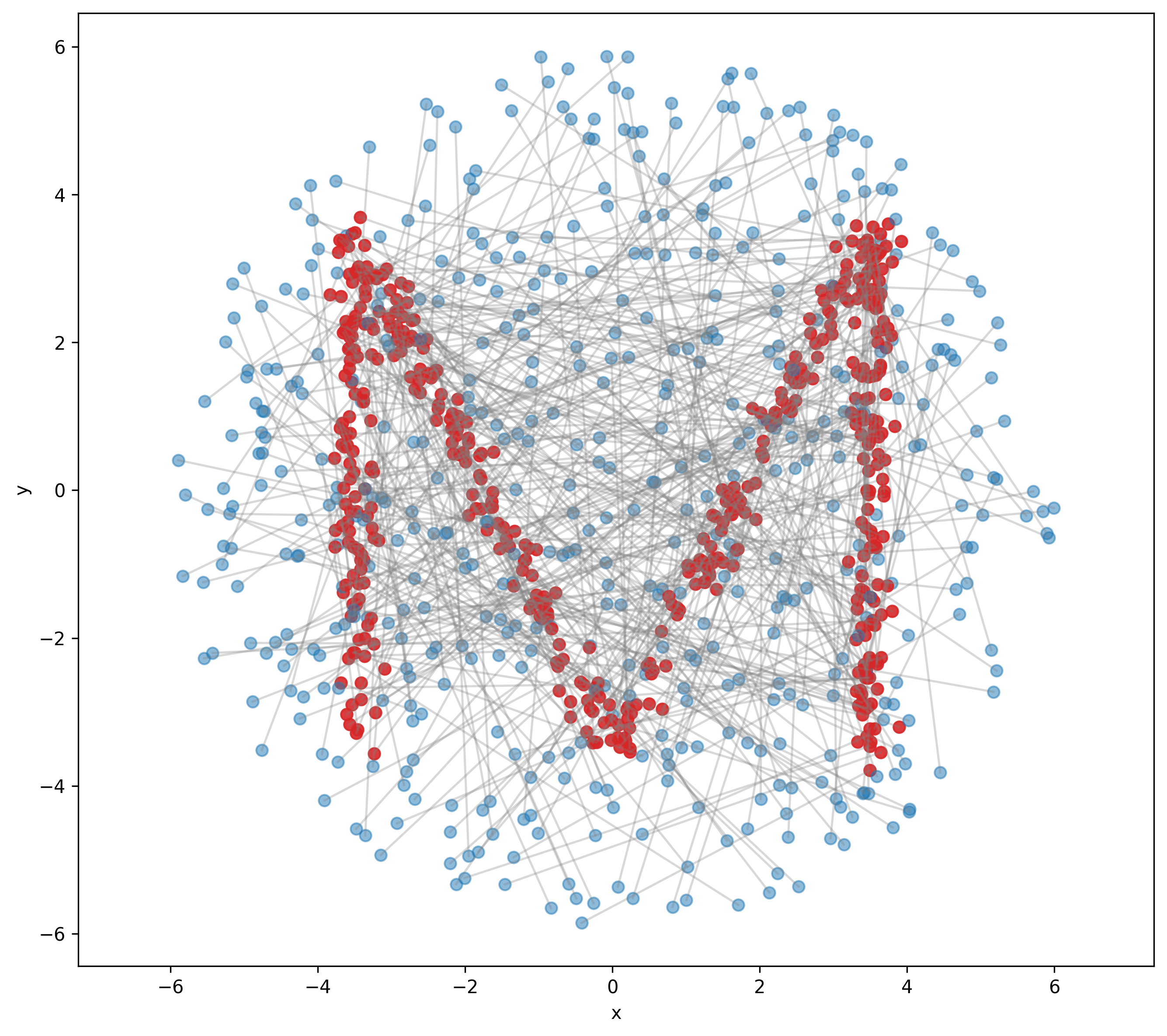}\\[-2pt]
      {\scriptsize Ground Truth}
    \end{minipage}\hfill
    \begin{minipage}[t]{0.33\linewidth}
      \centering
      \includegraphics[width=\linewidth]{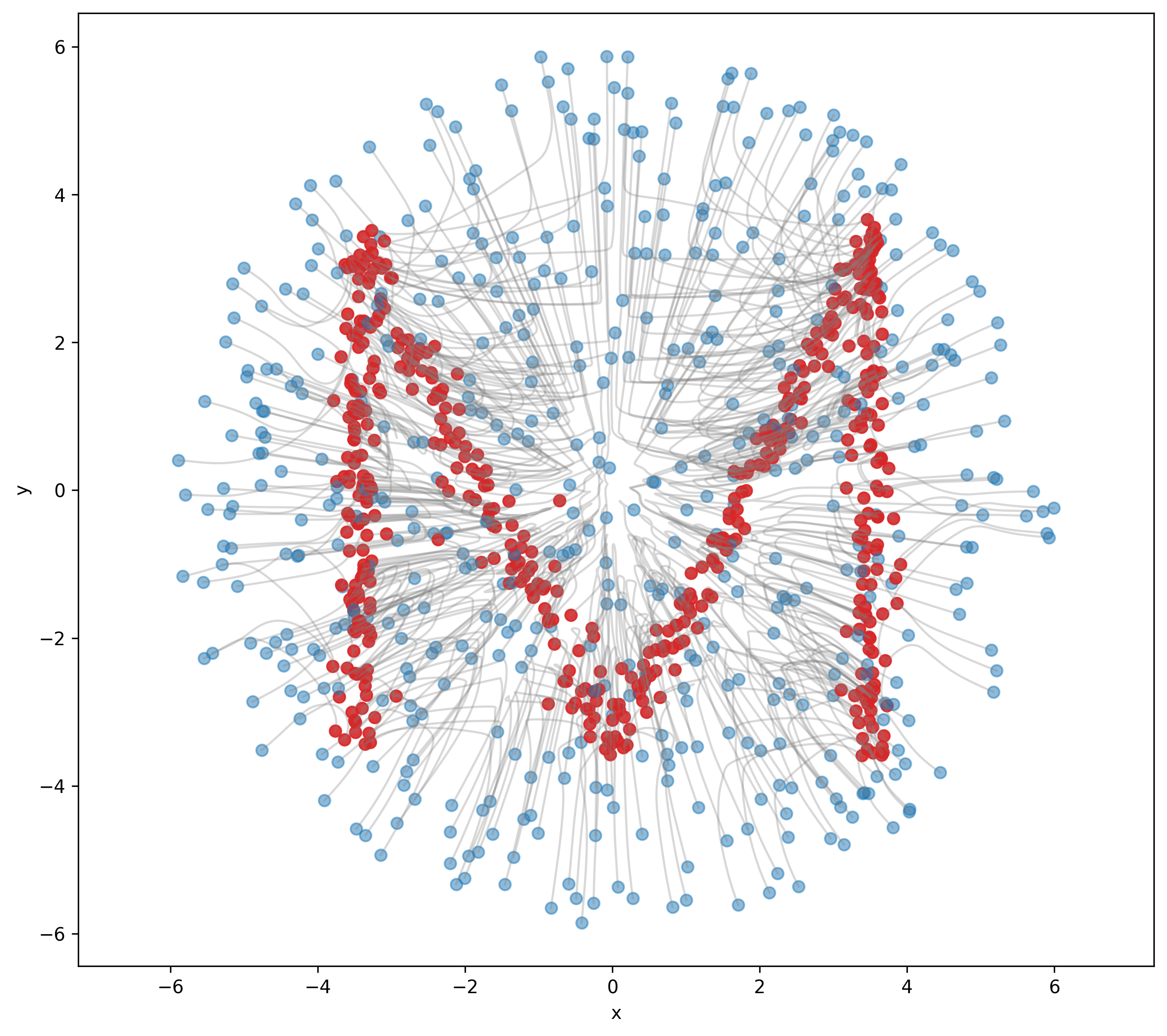}\\[-2pt]
      {\scriptsize Flow Matching}
    \end{minipage}\hfill
    \begin{minipage}[t]{0.33\linewidth}
      \centering
      \includegraphics[width=\linewidth]{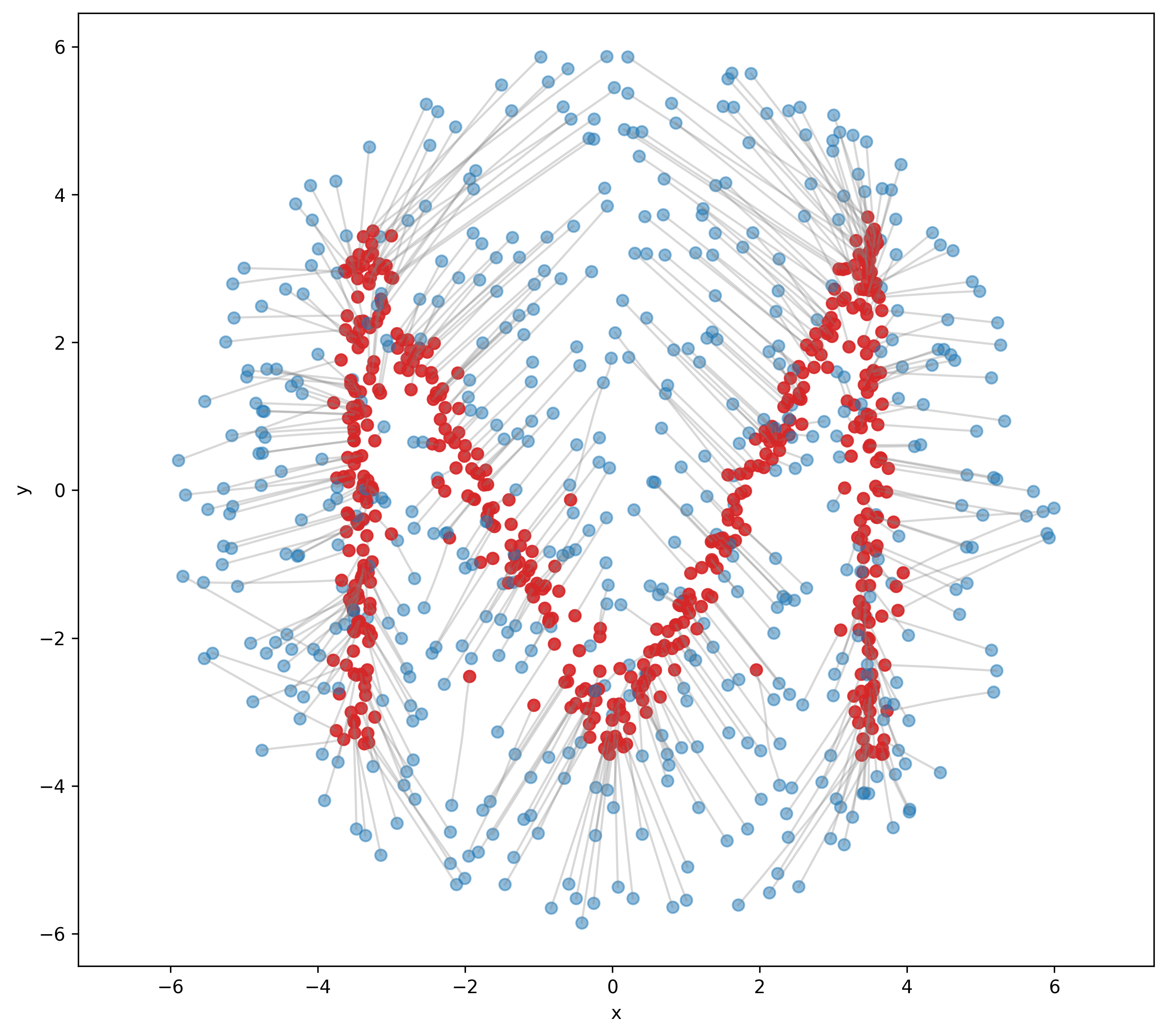}\\[-2pt]
      {\scriptsize Rectified Flow}
    \end{minipage}
  \end{minipage}\hfill
  % Block 2
  \begin{minipage}[t]{0.49\textwidth}
    \centering
    \begin{minipage}[t]{0.33\linewidth}
      \centering
      \includegraphics[width=\linewidth]{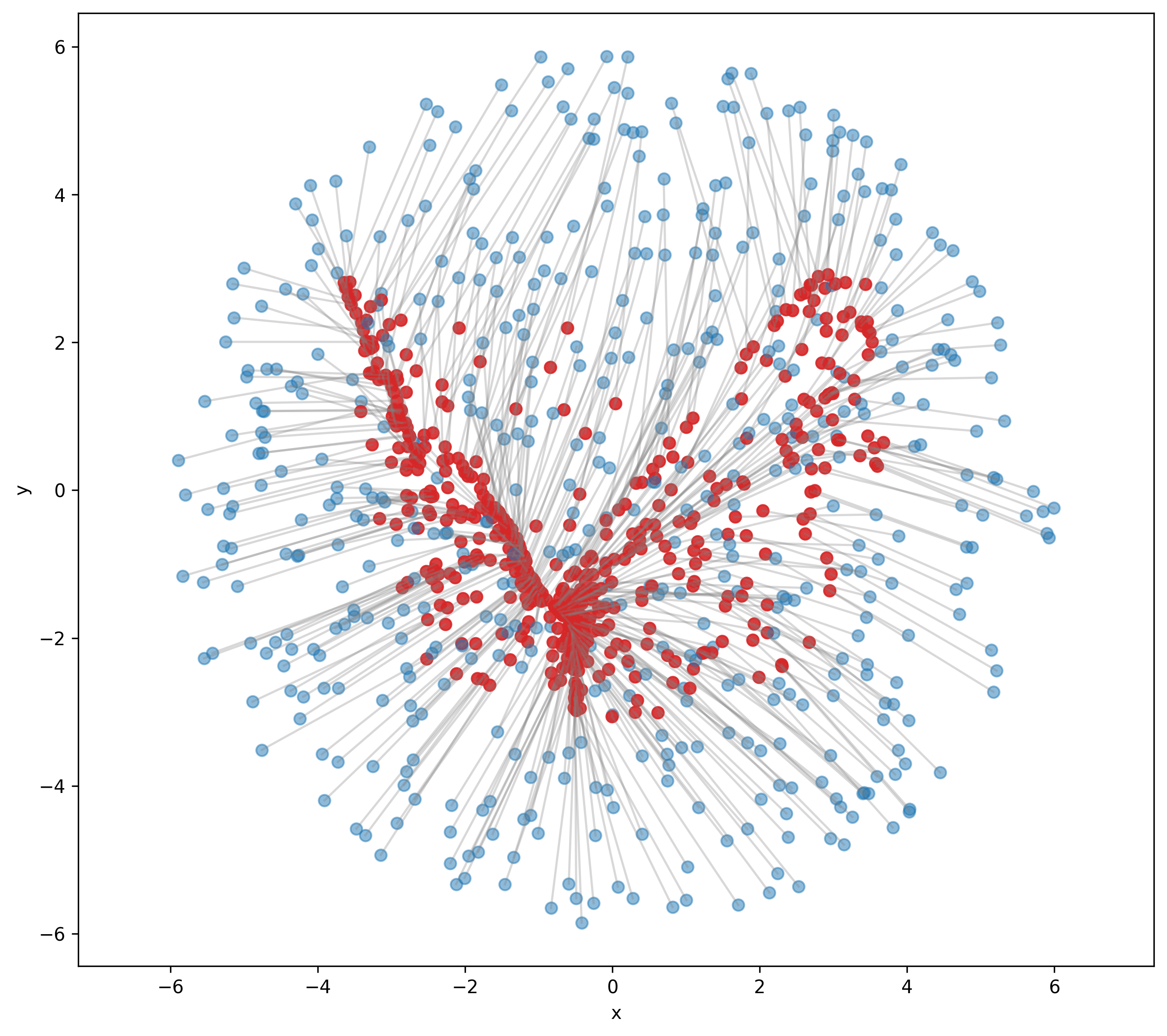}\\[-2pt]
      {\scriptsize Flow map (steps=1)}
    \end{minipage}\hfill
    \begin{minipage}[t]{0.33\linewidth}
      \centering
      \includegraphics[width=\linewidth]{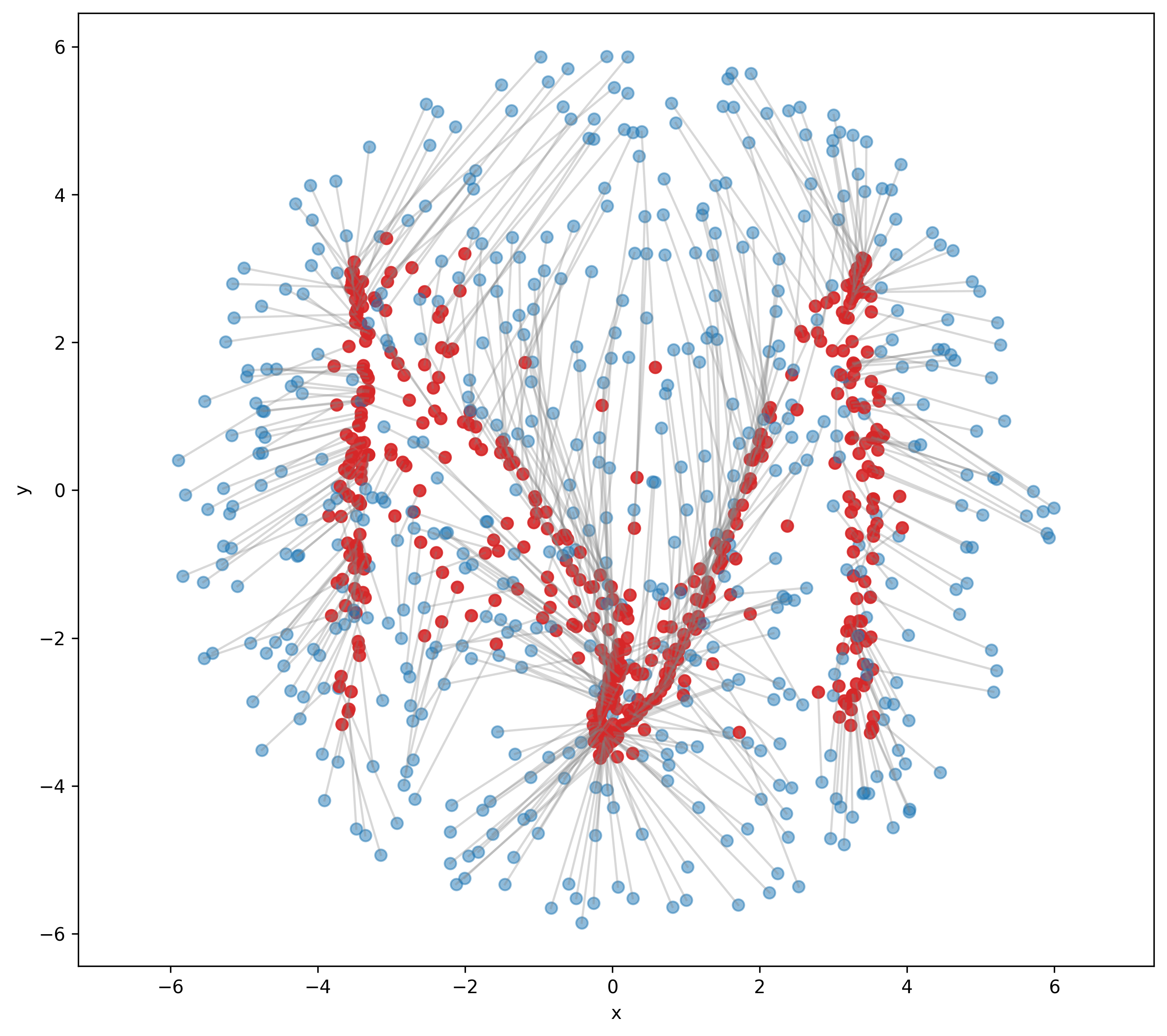}\\[-2pt]
      {\scriptsize MeanFlow (steps=1)}
    \end{minipage}\hfill
    \begin{minipage}[t]{0.33\linewidth}
      \centering
      \includegraphics[width=\linewidth]{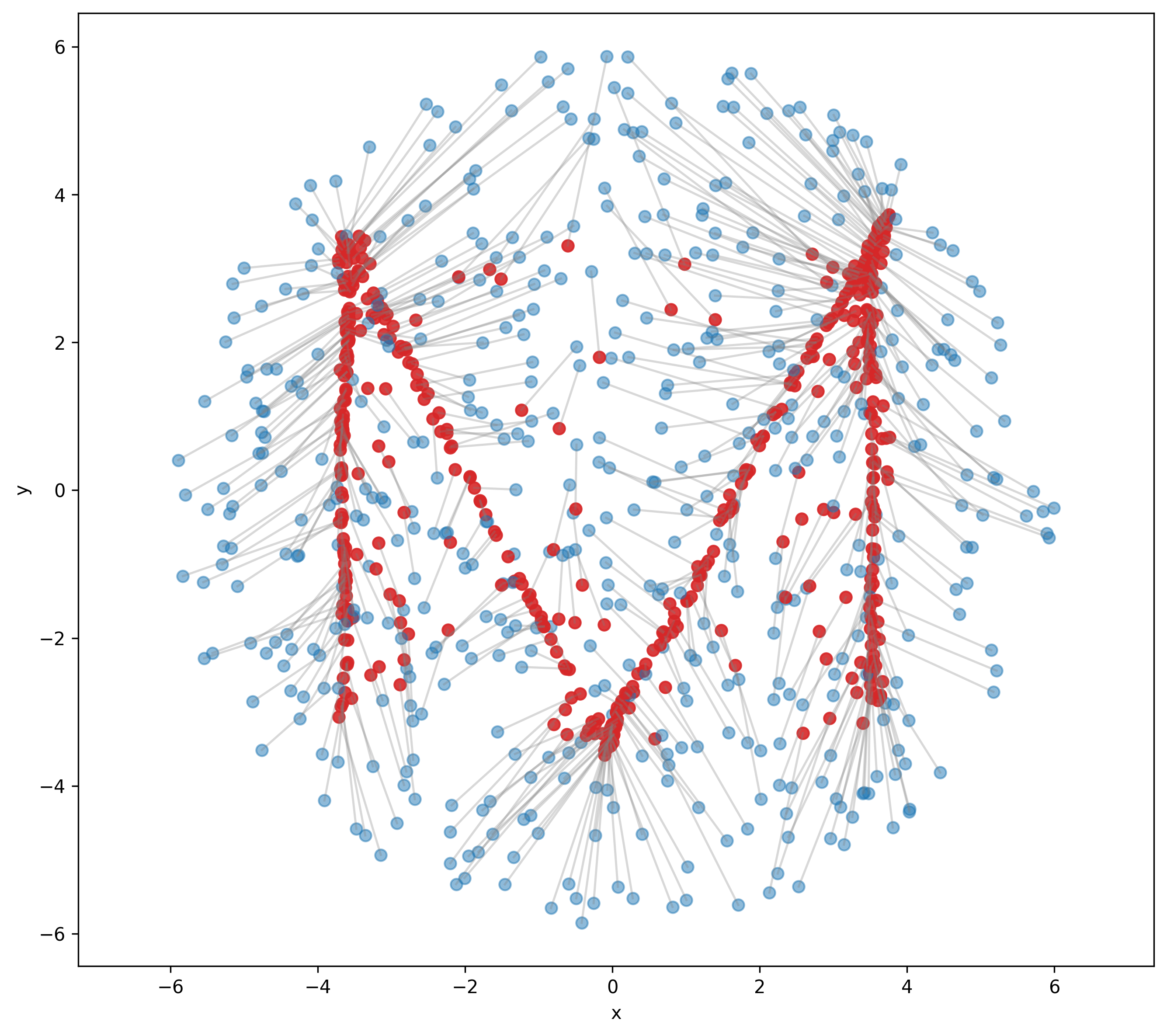}\\[-2pt]
      {\scriptsize TFM (steps=1)}
    \end{minipage}
  \end{minipage}

  \vspace{-1em}

  % ---------------- Row 2: Block 3 and Block 4 ----------------
  % Block 3
  \begin{minipage}[t]{0.49\textwidth}
    \centering
    \begin{minipage}[t]{0.33\linewidth}
      \centering
      \includegraphics[width=\linewidth]{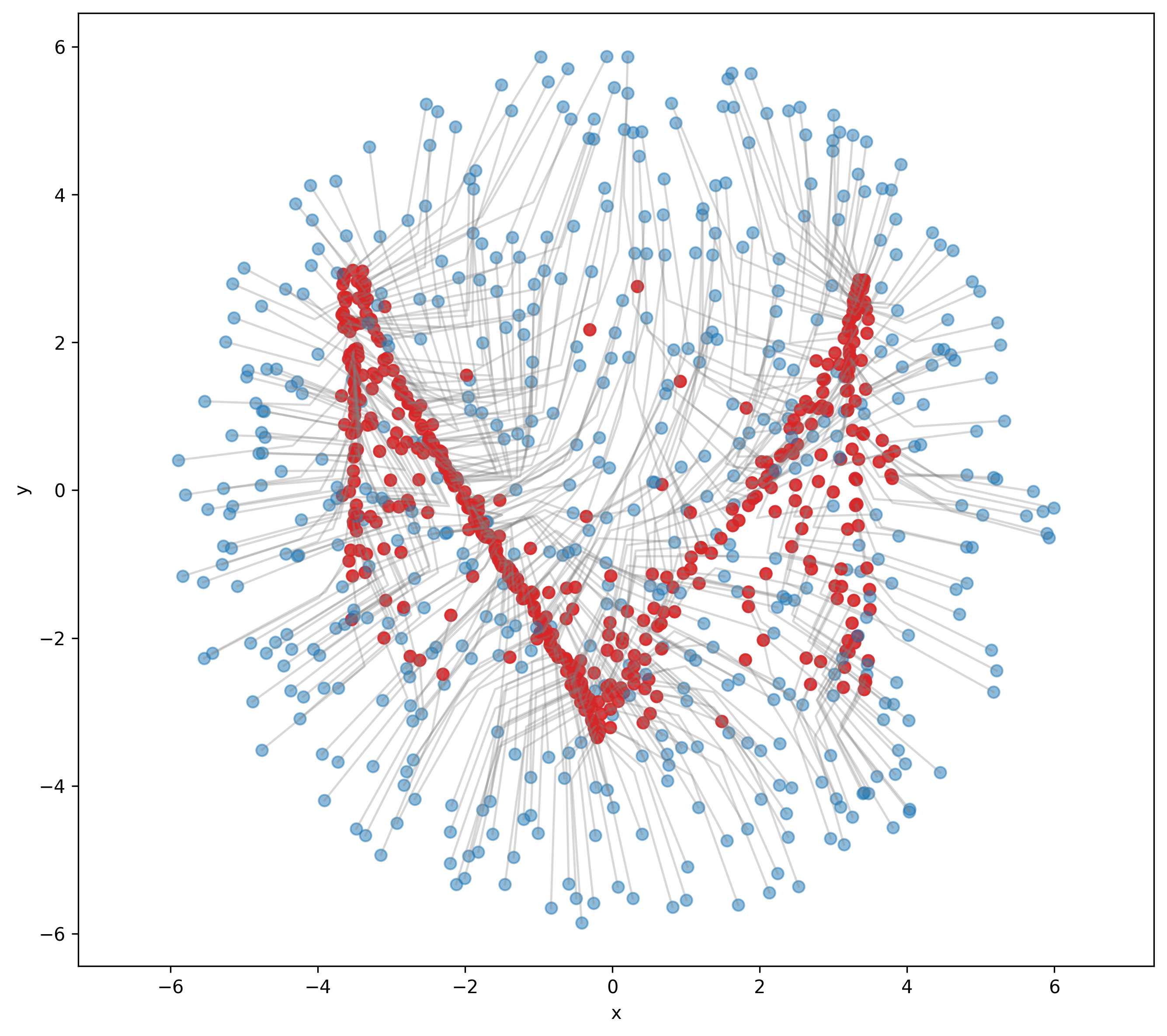}\\[-2pt]
      {\scriptsize Flow map (steps=2)}
    \end{minipage}\hfill
    \begin{minipage}[t]{0.33\linewidth}
      \centering
      \includegraphics[width=\linewidth]{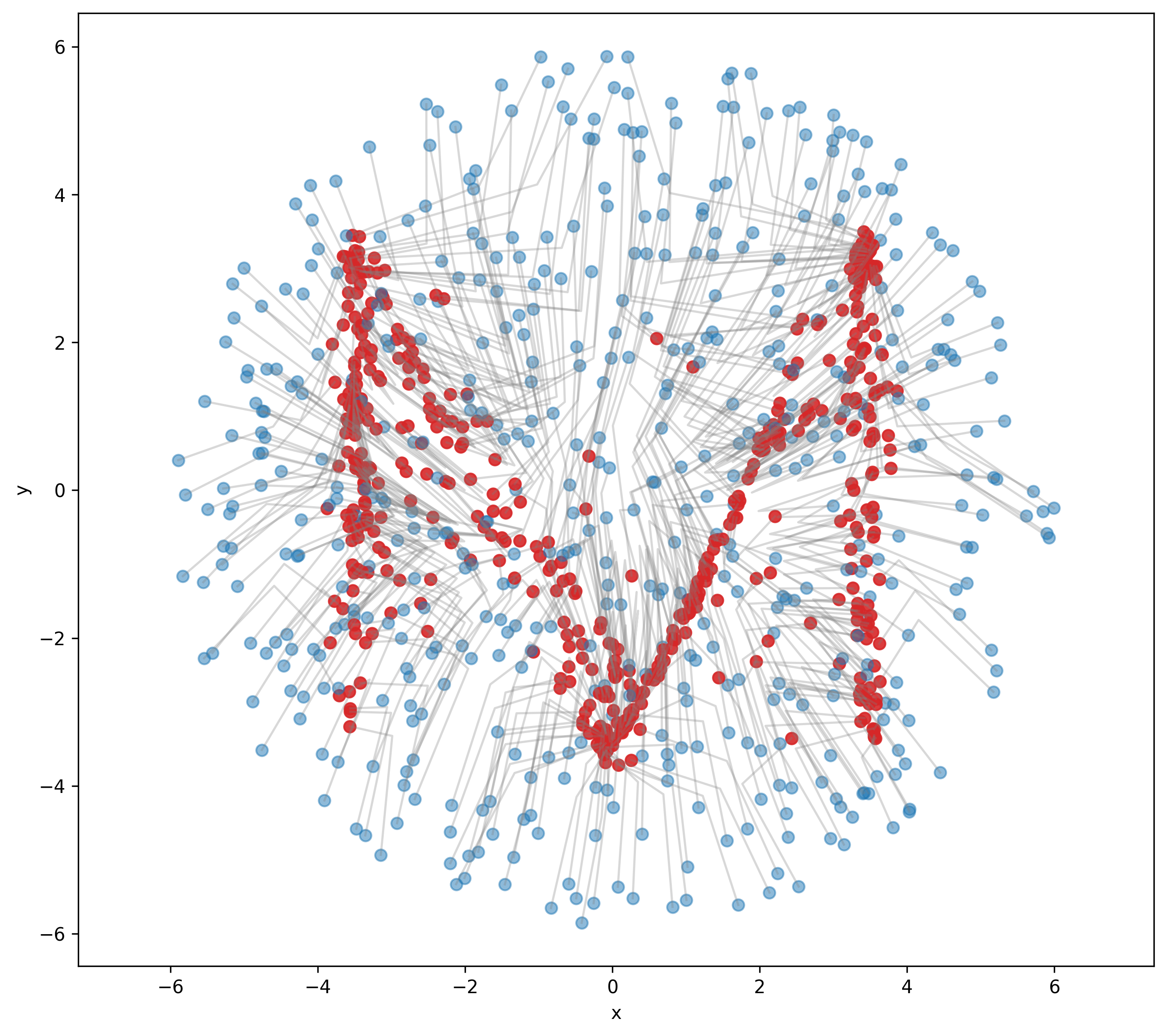}\\[-2pt]
      {\scriptsize MeanFlow (steps=2)}
    \end{minipage}\hfill
    \begin{minipage}[t]{0.33\linewidth}
      \centering
      \includegraphics[width=\linewidth]{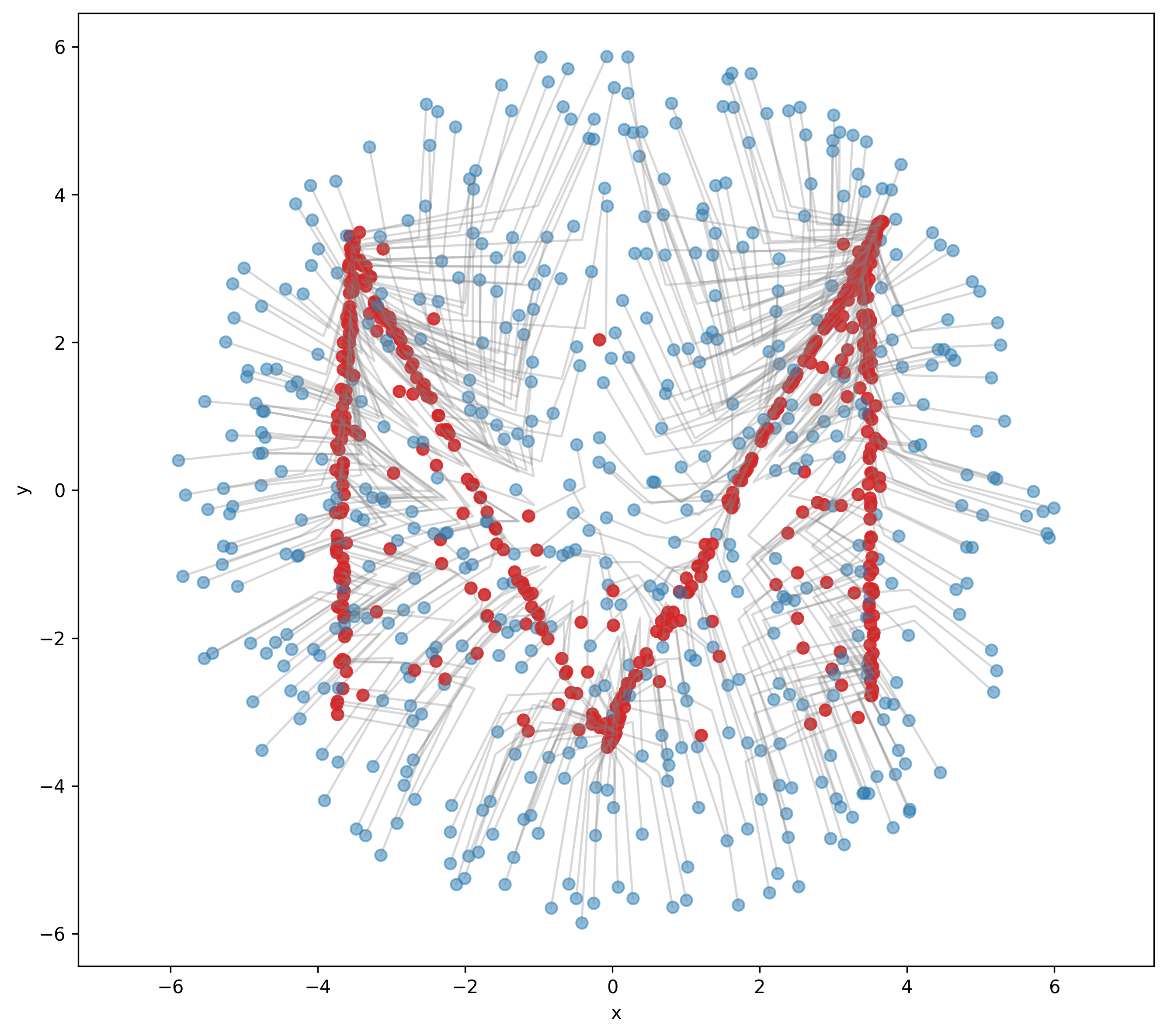}\\[-2pt]
      {\scriptsize TFM (steps=2)}
    \end{minipage}
  \end{minipage}\hfill
  % Block 4
  \begin{minipage}[t]{0.49\textwidth}
    \centering
    \begin{minipage}[t]{0.33\linewidth}
      \centering
      \includegraphics[width=\linewidth]{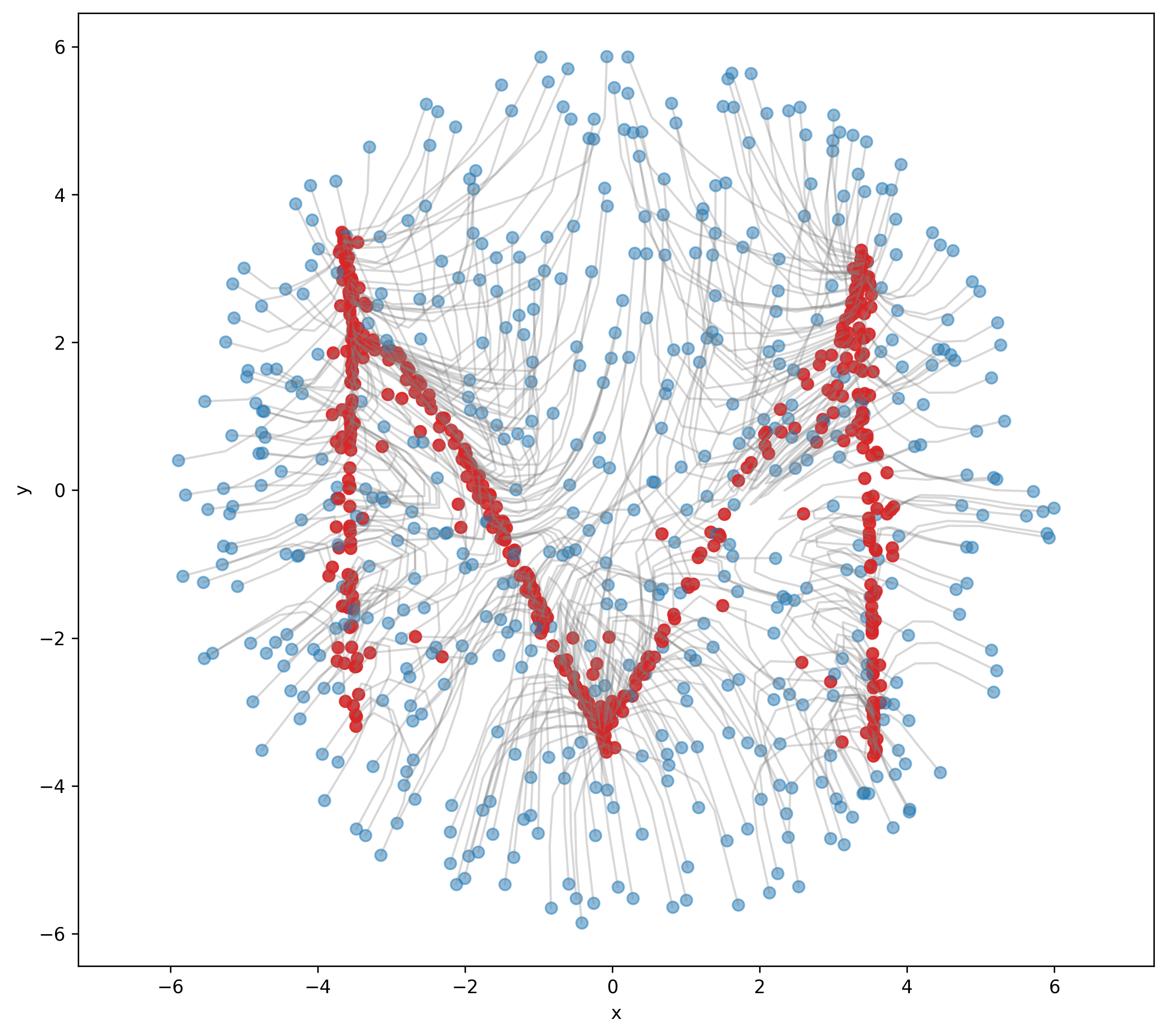}\\[-2pt]
      {\scriptsize Flow map (steps=5)}
    \end{minipage}\hfill
    \begin{minipage}[t]{0.33\linewidth}
      \centering
      \includegraphics[width=\linewidth]{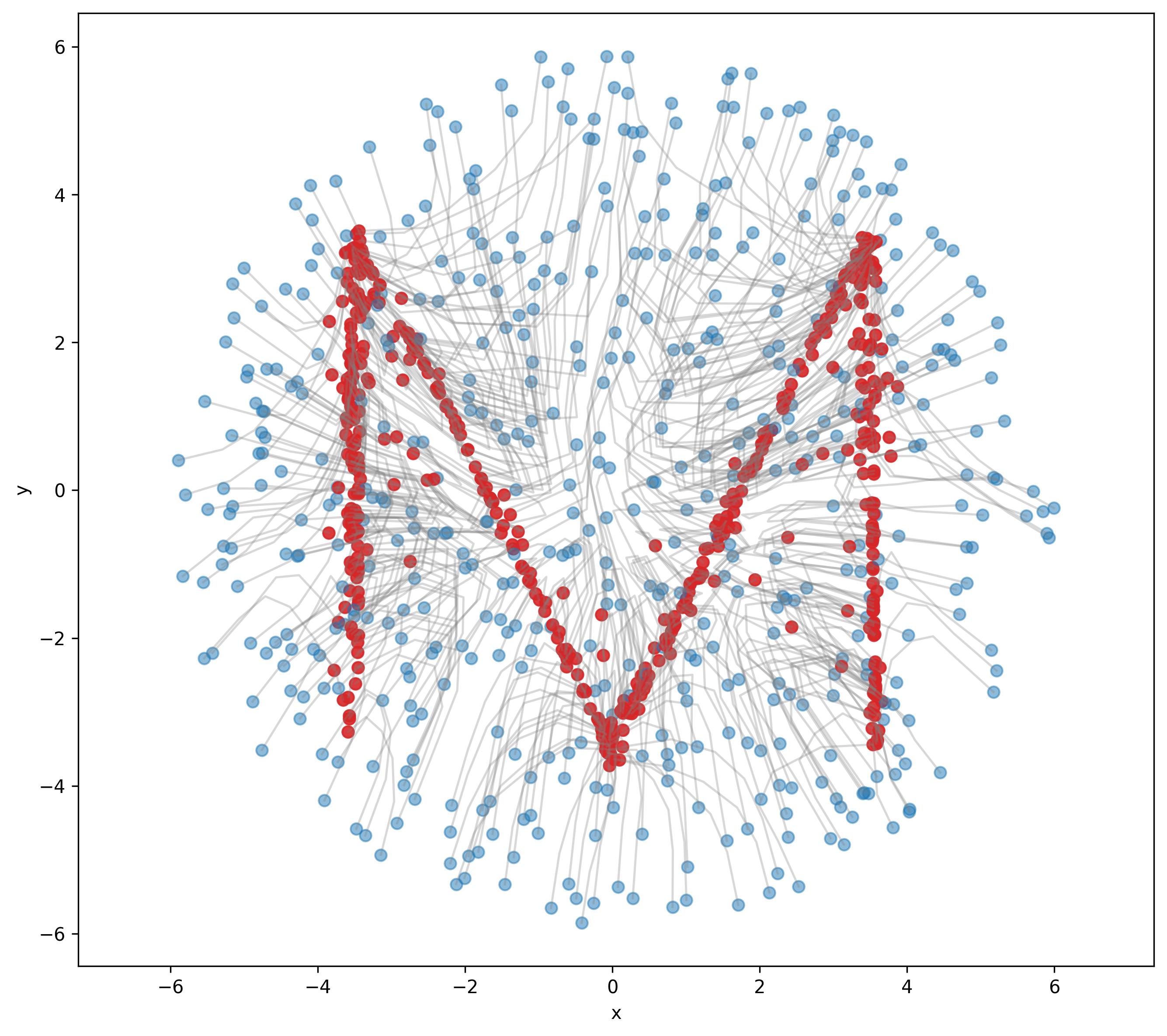}\\[-2pt]
      {\scriptsize MeanFlow (steps=5)}
    \end{minipage}\hfill
    \begin{minipage}[t]{0.33\linewidth}
      \centering
      \includegraphics[width=\linewidth]{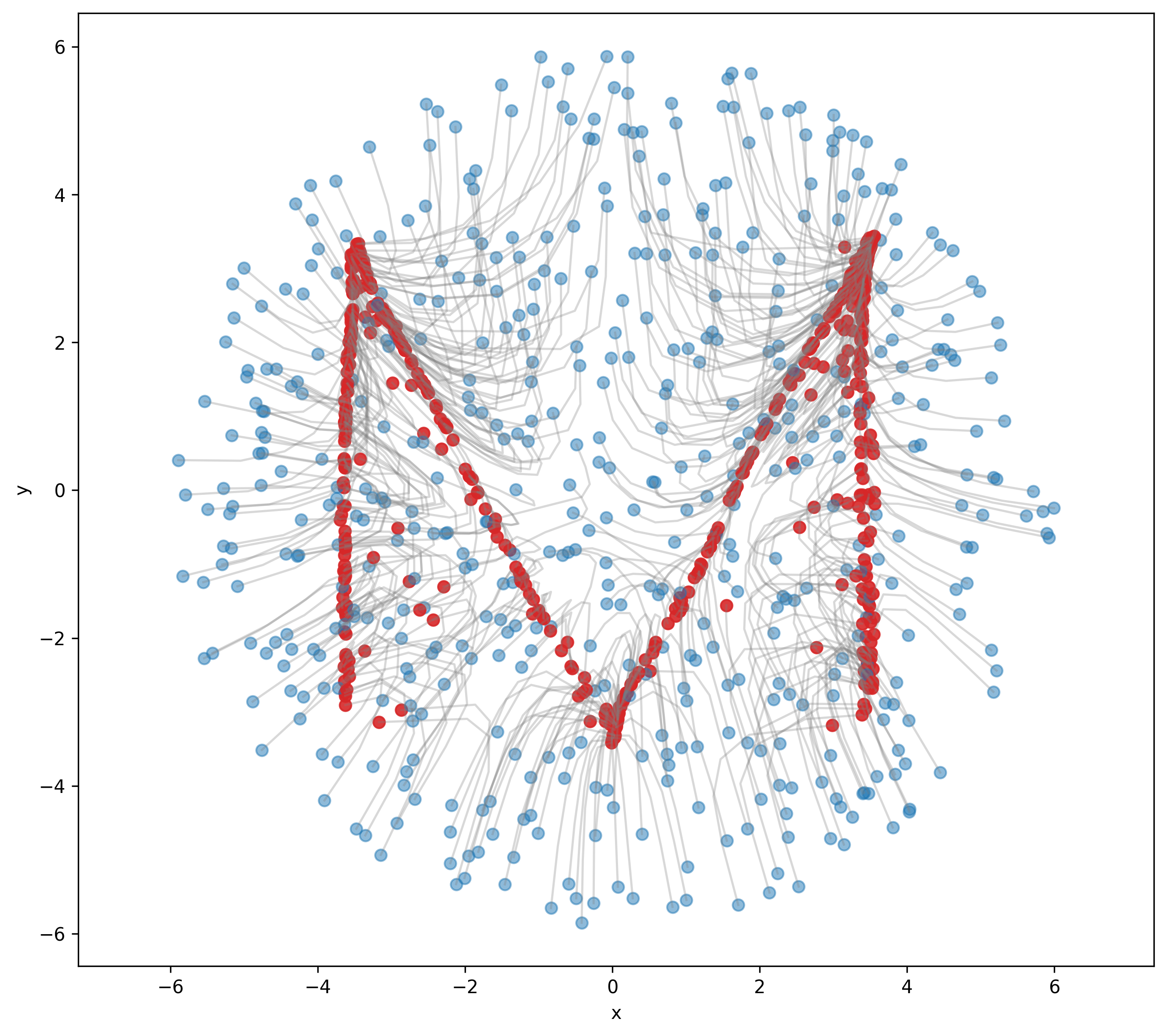}\\[-2pt]
      {\scriptsize TFM (steps=5)}
    \end{minipage}
  \end{minipage}

  \vspace{-1em}

  \caption{\emph{2D Generation Trajectory Visualization on Synthetic Data.} Tested methods include: Flow Matching\cite{lipman2023flow}, Rectified Flow\cite{liu2023flow}, Flow Map Matching\cite{boffi2024flow}, MeanFlow\cite{geng2025mean}.}
  \label{fig:image-viz}
  \vspace{-1em}
\end{figure*}

\paragraph{Loss Metrics.}
In Eq.~\eqref{eq:STFM}, the Bregman divergence $D(\cdot,\cdot)$ is instantiated as the squared $\ell_2$ loss. Following~\cite{geng2025mean}, we further investigate alternative loss metrics. In general, we consider loss functions of the form
$
\mathcal{L} = \|\Delta\|_2^{2\gamma}
$,
where $\Delta$ denotes the regression error. It can be shown (see~\cite{geng2025consistency}) that minimizing $\|\Delta\|_2^{2\gamma}$ is equivalent to minimizing the squared $\ell_2$ loss $\|\Delta\|_2^2$ with \emph{adapted loss weights}. 
% The detailed derivation is provided in the appendix. 
In practice, we define the weight as
$
w = \frac{1}{(\|\Delta\|_2^2 + c)^p}
$,
where $p = 1 - \gamma$ and $c > 0$ is a small constant (e.g., $10^{-3}$) for numerical stability. The resulting adaptively weighted loss takes the form $\mathrm{sg}(w)\cdot \mathcal{L}$, with $\mathcal{L} = \|\Delta\|_2^2$. When $p = 0.5$, this formulation resembles the Pseudo-Huber loss~\cite{song2024improved}. We compare different choices of $p$ in the experiments.

\paragraph{Sampling Time Steps $(t,r)$.}
We sample the two time steps $(t,r)$ from a predefined logit-normal (lognorm) distribution~\cite{esser2024scalingrectifiedflowtransformers, geng2025mean}. Specifically, we first draw a sample from a normal distribution $\mathcal{N}(\mu,\sigma)$ and map it to the interval $(0,1)$ via the logistic function to obtain $t$. We then sample another logit-normal variable $d$ in the same manner and set
$
r = t + d\,(1 - t)
$,
which ensures that the constraint $0 \leq t \leq r \leq 1$ is satisfied. Note that for any given $t$, the transition time $r$ is obtained by an affine transformation of a logit-normal random variable, ensuring that $r \in [t,1]$ while preserving a logit-normal–shaped density over the valid interval. Different hyperparameter settings of the logit-normal distribution are evaluated in the experiments.

\paragraph{Conditioning on $(t,r)$.}
We employ positional embeddings~\cite{geng2025mean, vaswani2023attentionneed} to encode the time variables, which are subsequently combined and used as conditioning inputs to the neural network. Although the vector field is parameterized as $X_\theta(x_t, t, r)$, it is not strictly necessary for the network to directly condition on $(t,r)$. For instance, the network can instead condition on $(t, \Delta t)$, where $\Delta t = r - t$. In this case, we define
$
X_\theta(\cdot, t, r) \triangleq \mathrm{net}(\cdot, t, r - t)
$,
with $\mathrm{net}$ denoting the neural network. The Jacobian-vector product (JVP) is always computed with respect to the function $X_\theta(\cdot, t, r)$. We empirically compare different conditioning strategies in the experiments.

\section{Experiments} \label{sec:experiment}
% In order to .

% In our ablation study, we use the EDM architecture as developed in [34], trained for 80 epochs (400K iterations) in the CIFAR-10 dataset for fast training and evaluation. 

\subsection{Synthetic Data and Visualization}
Synthetic data experiments are conducted to visualize the results and intuitively demonstrate the effectiveness of the proposed method~\cite{zhang2025hierarchical, ma2025learningstraightflowsvariational}. We simulate and visualize the generation trajectories of different methods on a 2D alphabet “M” dataset, as shown in Figure~\ref{fig:image-viz}. In this dataset, the source distribution (blue points) is circular, while the target distribution (red points) forms the shape of the letter “M”. The visualization results are consistent with the discussion in the previous section: our goal is to enable generation with arbitrary step sizes and an arbitrary number of steps by modeling the generation trajectory. Notably, even one-step generation shows promising results.
% Furthermore, we conduct experiments on a more complex synthetic dataset, as illustrated in Figure~\ref {fig:viz_2D_moon}. In this case, the source distribution (blue points) consists of eight Gaussian clusters surrounding the target distribution (red points), which forms an upper–lower moon shape. These results further demonstrate the robustness and generalization capability of our method in modeling complex nonlinear transformations.

\subsection{Visual Generation}

\paragraph{CIFAR-10.}
CIFAR-10 is a $32 \times 32$ resolution image dataset containing multiple classes and is a widely used benchmark in generative modeling~\cite{krizhevsky2009learning}. 
For a fair evaluation, we adopt the same UNet architecture and training protocol as in prior work~\cite{guo2025variational, geng2025consistency}, while replacing the conventional flow matching objective with the proposed \emph{Transition Flow Matching} objective as defined in Eq.~\eqref{eq:STFM}.

The UNet model $X^\theta$ follows a standard encoder--decoder design with residual blocks and skip connections. 
A self-attention block is inserted after the residual block at $16\times16$ resolution and at the bottleneck layer. 
The model takes the current state $x_t$ and the time variables $(t,r)$ as input, where $(t,r)$ are embedded and used to modulate adaptive group normalization layers through learnable scale and shift parameters.

\begin{table*}[t]
\centering
\scriptsize
\setlength{\tabcolsep}{4pt}
\renewcommand{\arraystretch}{1}

% adjustbox 在 table* 内部
\begin{adjustbox}{width=0.9\textwidth,center}
\begin{tabular}{c l c c c c c}
\Xhline{4\arrayrulewidth}
 & \textbf{NFE / Sampler} & \textbf{\# Params.} & \textbf{1} & \textbf{2} & \textbf{5} & \textbf{10} \\
\hline\hline
% ... (表格内容) ...
\multirow{2}{*}{\rotatebox[origin=c]{90}{\tiny \makecell[c]{Flow}}}
& {Flow Matching~\cite{lipman2023flow} \textcolor{lightgray}{\scriptsize{[ICLR'23]}}} & 36.5M & - & 166.65 & 36.19 & 14.4 \\
& {VFM~\cite{guo2025variational} \textcolor{lightgray}{\scriptsize{[ICML'25]}}} & 60.6M & - & 97.83 & 13.12 & 5.34 \\
\hline
% ... (表格内容) ...
\multirow{3}{*}{\rotatebox{90}{\tiny \makecell[c]{Re-Flow}}} 
& {1-Rectified Flow~\cite{liu2023flow} \textcolor{lightgray}{\scriptsize{[ICLR'23]}}} & 36.5M & 378 & 6.18 & - & - \\
& {2-Rectified Flow~\cite{liu2023flow} \textcolor{lightgray}{\scriptsize{[ICLR'23]}}} & 36.5M & 12.21 & 4.85 & - & - \\
& {3-Rectified Flow~\cite{liu2023flow} \textcolor{lightgray}{\scriptsize{[ICLR'23]}}} & 36.5M & 8.15 & 5.21 & - & - \\
\hline
% ... (表格内容) ...
\multirow{7}{*}{\rotatebox{90}{\tiny \makecell[c]{Mean Velocity/\\Consistency}}}
& {CT~\cite{song2023consistency} \textcolor{lightgray}{\scriptsize{[ICML'23]}}} & 61.8M & 8.71 & 5.32 & 11.412 & 23.948 \\
& {iCT~\cite{song2024improved} \textcolor{lightgray}{\scriptsize{[ICLR'24]}}} & ~55M & 2.83 & 2.46 & - & - \\
& {ECT~\cite{geng2025consistency} \textcolor{lightgray}{\scriptsize{[ICLR'25]}}} & ~55M & 3.60 & 2.11 & - & - \\
& {sCT~\cite{lu2025simplifying} \textcolor{lightgray}{\scriptsize{[ICLR'25]}}} & ~55M & 2.85 & 2.06 & - & - \\
& {IMM~\cite{zhou2025inductive} \textcolor{lightgray}{\scriptsize{[ICML'25]}}} & ~55M & 3.20 & 1.98 & - & - \\
& {MeanFlow~\cite{geng2025mean} \textcolor{lightgray}{\scriptsize{[NeurIPS'25]}}} & ~55M & 2.92 & 2.23 & 2.84 & 2.27 \\
& {S-VFM~\cite{ma2025learningstraightflowsvariational} \textcolor{lightgray}{\scriptsize{[CVPR'26]}}} & 60.6M & 2.81 & 2.16 & 2.02 & 1.97 \\
\hline
% ... (表格内容) ...
\rowcolor{cvprblue!15} 
& {TFM \textcolor{lightgray}{\scriptsize{[Ours]}}} & 55M & 2.77 & 2.08 & 1.96 & 1.91 \\
\Xhline{4\arrayrulewidth}
\end{tabular}
\end{adjustbox} % <-- adjustbox 在 table* 内部结束

\vspace{0em}
\caption{
\emph{Quantitative Comparison with Different Generation Methods on CIFAR-10 Dataset.} 
Our method achieves the best performance in one-step generation (NFE $=1$). 
Moreover, the FID score consistently decreases as NFE increases.
}
\label{tab:cifar10_fid}
\vspace{-2em}
\end{table*}

To quantitatively evaluate the generation performance, we compare our method with several state-of-the-art approaches by measuring the generation quality using the Fréchet Inception Distance (FID)~\cite{heusel2017gans}, computed under varying NFE: $[1, 2, 5, 10]$ and adaptive-step Dopri5 ODE solver~\cite{dormand1980family}, as presented in Table~\ref{tab:cifar10_fid}. 
The results in Table~\ref{tab:cifar10_fid} show that our method achieves the best performance in one-step generation (NFE $=1$). 
Moreover, the FID score consistently decreases as NFE increases, while both Consistency Models and Mean Velocity Models tend to exhibit degraded performance with higher NFE values.

% During training, we sample $(X_0,X_1)$ pairs from the data distribution and construct linear interpolants following the standard Flow Matching setting. 
% Given a randomly sampled time pair $(t,r)$ with $0\le t\le r\le1$, the network is trained to regress the transition flow using the tractable Transition Flow Matching objective described in Eq.~\eqref{eq:STFM}, which enforces consistency with the Transition Flow Identity.

% Figure~\ref{} illustrates the generation results under different numbers of function evaluations (NFE), using two distinct initial noise sets $X_0^1$ and $X_0^2$. 
% Within each row, images at the same spatial location across different panels are generated from the same initial noise realization. 
% Each column corresponds to a different NFE setting, with NFE values of $[1,2,5,10]$ from left to right. 
% As the number of function evaluations increases, the generated images exhibit improved visual fidelity and structural coherence. 
% Notably, even single-step generation (NFE $=1$) produces competitive samples, supporting the claim that the transition flow model $X^\theta(X_t,t,r)$ learns the transition to the state at $r$ from state $x_t$ at $t$, enabling effective few-step sampling.

\paragraph{ImageNet.}
To evaluate robustness and scalability on large-scale data, we conduct experiments on the ImageNet dataset with image resolution $256\times256$~\cite{krizhevsky2012imagenet}. 
All experiments are performed on class-conditional ImageNet generation at this resolution. 
Following common practice, we evaluate the Fr\'echet Inception Distance (FID)~\cite{heusel2017gans} on 50K randomly generated images.
\begin{figure*}[t]
\centering

% Local box settings (only affect this figure)
{
\setlength{\fboxrule}{1.1pt}  % frame line width
\setlength{\fboxsep}{1.0pt}   % padding between frame and content

% A small helper width for the left text label (tune if needed)
\newcommand{\labw}{0.08\textwidth}

% ---------- Row 1 ----------
\makebox[\textwidth][c]{%
\fcolorbox{blue!80!black}{white}{%
\makebox[\labw][l]{\color{blue!80!black}\scriptsize\textbf{NFE=1}}\hspace{3.0mm}%
\includegraphics[width=0.108\textwidth]{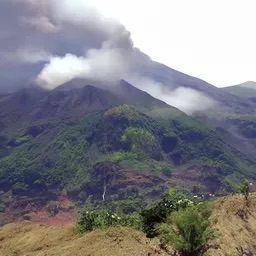}\hspace{0.25mm}%
\includegraphics[width=0.108\textwidth]{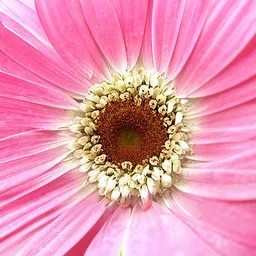}\hspace{0.25mm}%
\includegraphics[width=0.108\textwidth]{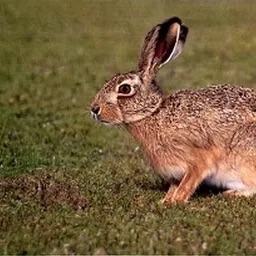}\hspace{0.25mm}%
\includegraphics[width=0.108\textwidth]{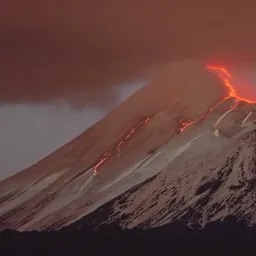}\hspace{0.25mm}%
\includegraphics[width=0.108\textwidth]{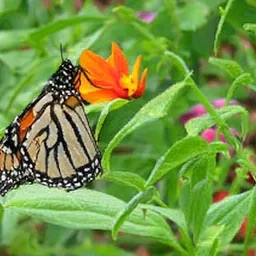}\hspace{0.25mm}%
\includegraphics[width=0.108\textwidth]{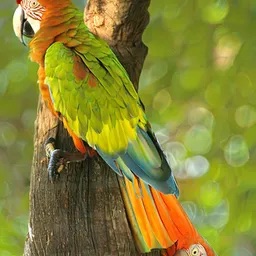}\hspace{0.25mm}%
\includegraphics[width=0.108\textwidth]{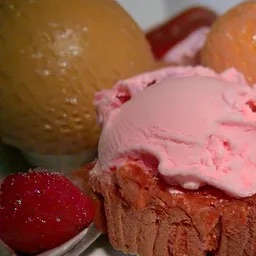}\hspace{0.25mm}%
\includegraphics[width=0.108\textwidth]{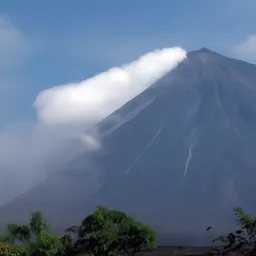}%
}}%

\vspace{-0.3mm}

% ---------- Row 2 ----------
\makebox[\textwidth][c]{%
\fcolorbox{blue!65!red}{white}{%
\makebox[\labw][l]{\color{blue!65!red}\scriptsize\textbf{NFE=2}}\hspace{3.0mm}%
\includegraphics[width=0.108\textwidth]{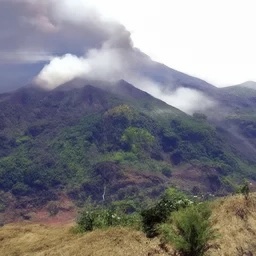}\hspace{0.25mm}%
\includegraphics[width=0.108\textwidth]{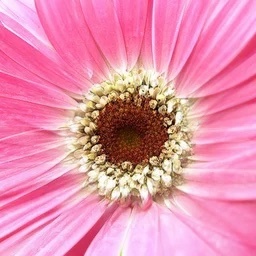}\hspace{0.25mm}%
\includegraphics[width=0.108\textwidth]{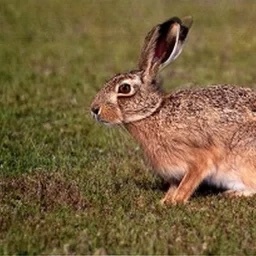}\hspace{0.25mm}%
\includegraphics[width=0.108\textwidth]{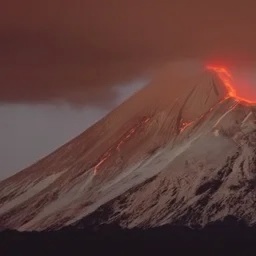}\hspace{0.25mm}%
\includegraphics[width=0.108\textwidth]{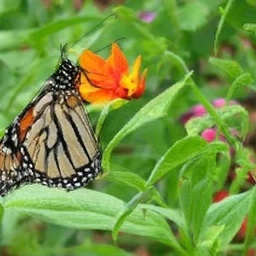}\hspace{0.25mm}%
\includegraphics[width=0.108\textwidth]{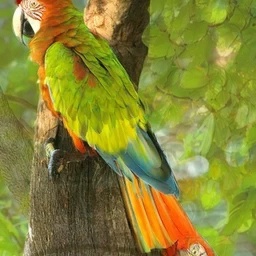}\hspace{0.25mm}%
\includegraphics[width=0.108\textwidth]{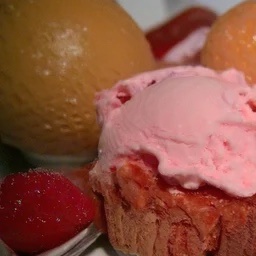}\hspace{0.25mm}%
\includegraphics[width=0.108\textwidth]{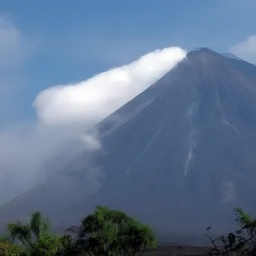}%
}}%

\vspace{-0.3mm}

% ---------- Row 3 ----------
\makebox[\textwidth][c]{%
\fcolorbox{red!65!blue}{white}{%
\makebox[\labw][l]{\color{red!65!blue}\scriptsize\textbf{NFE=5}}\hspace{3.0mm}%
\includegraphics[width=0.108\textwidth]{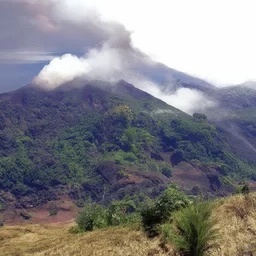}\hspace{0.25mm}%
\includegraphics[width=0.108\textwidth]{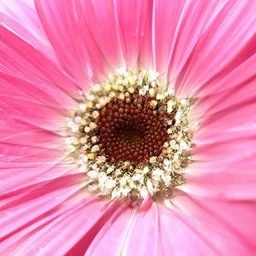}\hspace{0.25mm}%
\includegraphics[width=0.108\textwidth]{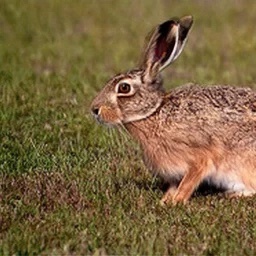}\hspace{0.25mm}%
\includegraphics[width=0.108\textwidth]{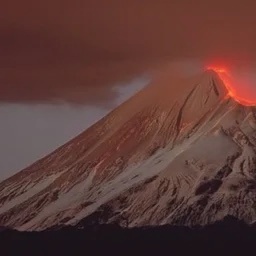}\hspace{0.25mm}%
\includegraphics[width=0.108\textwidth]{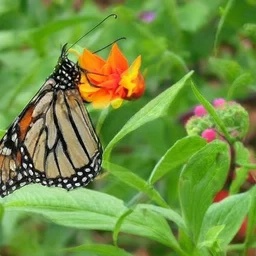}\hspace{0.25mm}%
\includegraphics[width=0.108\textwidth]{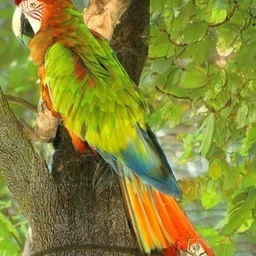}\hspace{0.25mm}%
\includegraphics[width=0.108\textwidth]{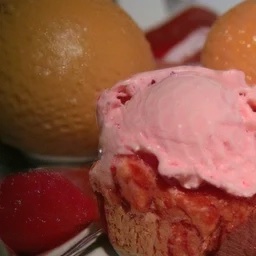}\hspace{0.25mm}%
\includegraphics[width=0.108\textwidth]{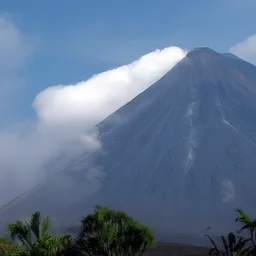}%
}}%

\vspace{-0.3mm}

% ---------- Row 4 ----------
\makebox[\textwidth][c]{%
\fcolorbox{red!85!black}{white}{%
\makebox[\labw][l]{\color{red!85!black}\scriptsize\textbf{NFE=10}}\hspace{3.0mm}%
\includegraphics[width=0.108\textwidth]{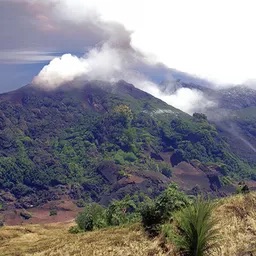}\hspace{0.25mm}%
\includegraphics[width=0.108\textwidth]{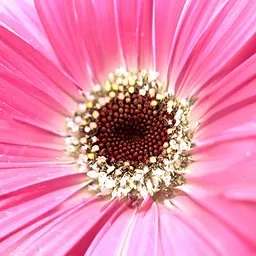}\hspace{0.25mm}%
\includegraphics[width=0.108\textwidth]{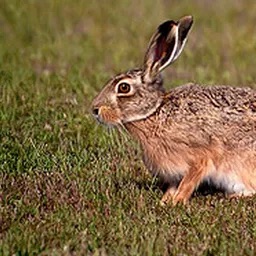}\hspace{0.25mm}%
\includegraphics[width=0.108\textwidth]{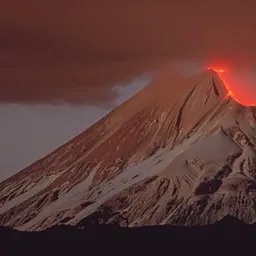}\hspace{0.25mm}%
\includegraphics[width=0.108\textwidth]{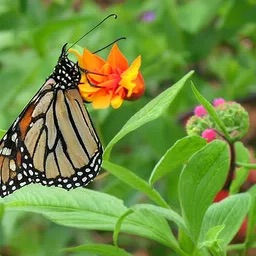}\hspace{0.25mm}%
\includegraphics[width=0.108\textwidth]{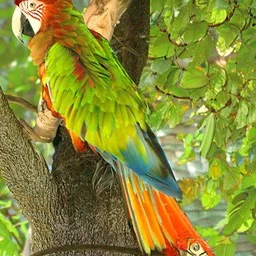}\hspace{0.25mm}%
\includegraphics[width=0.108\textwidth]{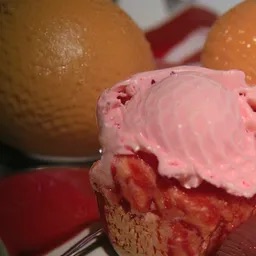}\hspace{0.25mm}%
\includegraphics[width=0.108\textwidth]{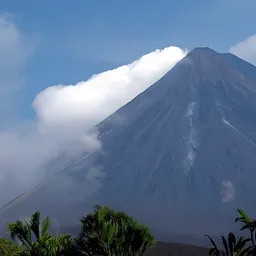}%
}}%

} % end local settings group

\vspace{-1em}

\caption{\emph{Generation Results on ImageNet-256 under Varying NFE.} As the number of function evaluations (NFE) increases from 1 to 10, the generated images exhibit progressively improved detail and fidelity. Notably, even single-step generation already produces reasonably good results.}
\label{fig:imagenet}
\vspace{-2em}
\end{figure*}
Following prior works~\cite{frans2025one, ma2025learningstraightflowsvariational, guo2025variational}, we implement all models in the latent space of a pre-trained VAE tokenizer~\cite{rombach2022highresolutionimagesynthesislatent}. 
For $256\times256$ images, the tokenizer maps images into a latent representation of size $32\times32\times4$, which serves as the input to the generative model. 
All models are trained from scratch under identical data and optimization settings. 
% Additional implementation details are provided in Appendix~\S\ref{}.

As the backbone architecture, we adopt MeanFlow~\cite{geng2025mean}, a transformer-based model that has demonstrated strong performance in high-resolution image generation. 
For fair comparison, we strictly follow the original MeanFlow~\cite{geng2025mean} training recipe and optimization settings, modifying only the learning objective.

Specifically, we introduce \textbf{TFM}, which is parameterized by $X^\theta(X_t,t,r)$, by replacing the original velocity-based loss with the proposed Transition Flow Matching loss in Eq.~\eqref{eq:STFM}. 
The model is trained to directly predict the transition flow $X(x_t,t,r)$ rather than the local velocity field. 
Both time variables $t$ and $r$ are embedded and injected into the transformer blocks via adaptive normalization layers.

During training, we sample $(X_0,X_1)$ pairs from the data distribution and construct linear interpolants following the standard Flow Matching setting. 
Given a randomly sampled time pair $(t,r)$ with $0\le t\le r\le1$, the network is trained to regress the transition flow using the tractable Transition Flow Matching objective described in Eq.~\eqref{eq:STFM}, which enforces consistency with the Transition Flow Identity.
At inference time, sample generation is performed by repeatedly applying the learned transition flow across a predefined time grid, starting from Gaussian noise. 
This formulation allows the model to directly predict future states along the transition trajectory, enabling flexible and efficient generation.

Figure~\ref{fig:imagenet} illustrates generation results under different numbers of function evaluations (NFE), using distinct initial noise. 
Within each row, images at the same spatial location across different panels are generated from the same initial noise realization. 
Each column corresponds to a different NFE setting, with NFE values of $[1,2,5,10]$ arranged from top to bottom.
As the number of function evaluations increases, the generated images exhibit improved visual fidelity and structural coherence. 
Notably, even one-step generation (NFE $=1$) produces competitive samples, supporting the claim that the transition flow model $X^\theta(X_t,t,r)$ learns to predict the future state at $r$ from the current state $X_t$ at time $t$, thereby enabling effective few-step sampling.

\begin{table}[t]
\centering
\scriptsize
\setlength{\tabcolsep}{5pt}
\renewcommand{\arraystretch}{1}

\begin{adjustbox}{width=0.7\linewidth,center}
\begin{tabular}{lcccc}
\Xhline{4\arrayrulewidth} % <-- 替换 \toprule
\textbf{Method} & \textbf{\# Params.} & \textbf{NFE} & \textbf{FID} \\
\hline\hline % <-- 替换 \midrule
iCT-XL/2~\cite{song2024improved} \textcolor{lightgray}{\scriptsize{[ICLR'24]}} & 675M & 1 & 34.24 \\
Shortcut-XL/2~\cite{frans2025one} \textcolor{lightgray}{\scriptsize{[ICLR'25]}} & 675M & 1 & 10.60 \\
MeanFlow-XL/2~\cite{geng2025mean} \textcolor{lightgray}{\scriptsize{[NeurIPS'25]}} & 676M & 1 & 3.43 \\
S-VFM-XL/2~\cite{ma2025learningstraightflowsvariational} \textcolor{lightgray}{\scriptsize{[CVPR'26]}} & 677M & 1 & 3.31 \\
\rowcolor{cvprblue!15}
TFM-XL/2 \textcolor{lightgray}{\scriptsize{[Ours]}} & 676M & 1 & 3.02 \\
\hline % <-- 替换 \midrule
iCT-XL/2~\cite{song2024improved} \textcolor{lightgray}{\scriptsize{[ICLR'24]}} & 675M & 2 & 20.30 \\
iMM-XL/2~\cite{zhou2025inductive} \textcolor{lightgray}{\scriptsize{[ICML'25]}} & 675M & $1\times2$ & 7.77 \\
MeanFlow-XL/2~\cite{geng2025mean} \textcolor{lightgray}{\scriptsize{[NeurIPS'25]}} & 676M & 2 & 2.93 \\
S-VFM-XL/2~\cite{ma2025learningstraightflowsvariational} \textcolor{lightgray}{\scriptsize{[CVPR'26]}} & 677M & 2 & 2.86 \\
\rowcolor{cvprblue!15}
TFM-XL/2 \textcolor{lightgray}{\scriptsize{[Ours]}} & 676M & 2 & 2.77 \\
\Xhline{4\arrayrulewidth} % <-- 替换 \bottomrule
\end{tabular}
\end{adjustbox} % 2. 关闭 adjustbox

\vspace{0em}
\caption{\emph{Quantitative Comparison with Different Generation Methods on ImageNet $256 \times 256$ Dataset.} 
Our method achieves the best performance in few-step generation.}
\label{tab:image256}
\vspace{-3em}
\end{table}

Following standard evaluation protocols, we randomly generate 50K images from each model and report the corresponding FID scores in Table~\ref{tab:image256}. 
TFM-XL consistently outperforms both Consistency Models and Mean Velocity Models under identical training settings. 
These results demonstrate that explicitly modeling \emph{global transition dynamics} through a transition flow—capable of predicting arbitrary future states—rather than learning a naturally local velocity field, provides greater flexibility in simulation steps for flow- and diffusion-based generation methods.

We further analyze the training dynamics by comparing TFM with SiT and MeanFlow across different training iterations, as shown in Figure~\ref{fig:train}. 
For SiT, we follow the default inference setting with $\text{NFE}=250$, while for both TFM and MeanFlow we use $\text{NFE}=1$ to reflect their one-step generation capability. 
The results show that TFM achieves clear performance improvements once sufficiently trained, and the training curves indicate that performance consistently improves as the number of training epochs increases. 
This behavior confirms the effectiveness and scalability of the proposed transition flow formulation in both training and inference.

\subsection{Ablation Study}

In our ablation study, we use the ViT-B/4 architecture~\cite{dosovitskiy2021imageworth16x16words} (“Base" size with a patch size of 4 as developed in ~\cite{ma2024sit}, trained for 80 epochs (400K iterations).

\paragraph{Conditioning on $(t,r)$.}
Following the model parameterization in Sec.~\ref{sec:implementation}, the transition flow $X_\theta(x_t,t,r)$ requires explicit conditioning on the temporal variables. Similar to prior designs, we encode time information through positional embeddings and study different conditioning strategies that share the same functional form but differ in the specific choice of variables. Concretely, instead of directly conditioning on $(t,r)$, the network can equivalently condition on $(t,\Delta t)$ with $\Delta t = r-t$, leading to the parameterization
$
X_\theta(\cdot,t,r)\triangleq \mathrm{net}(\cdot,t,r-t).
$
We compare these variants in Tab.~\ref{tab:pos}. The results indicate that all studied conditioning forms lead to stable and effective one-step generation, demonstrating that our transition flow formulation is robust to the exact choice of temporal embedding. Conditioning on $(t,\Delta t)$ yields the strongest performance overall, while directly using $(t,r)$ performs comparably. Notably, even conditioning solely on the interval $\Delta t$ produces competitive results, suggesting that relative temporal information plays a dominant role in our method.

\begin{table*}[t]
\centering

\begin{subtable}{0.3\linewidth}
\centering
\begin{tabular}{cc}
\hline
pos. embed & FID, 1-NFE \\
\hline
$(t,r)$ & 60.12 \\
$(t,t-r)$ & \textbf{59.87} \\
$(t,r,t-r)$ & 62.48 \\
$t-r$ only & 62.16 \\
\hline
\end{tabular}
\caption{\emph{Positional embedding.} The network is conditioned on the embeddings applied to the specified variables.}
\label{tab:pos}
\end{subtable}
\hfill
\begin{subtable}{0.3\linewidth}
\centering
\begin{tabular}{cc}
\hline
$p$ & FID, 1-NFE \\
\hline
0.0 & 79.76 \\
0.5 & 62.47 \\
1.0 & \textbf{59.87} \\
1.5 & 65.68 \\
2.0 & 69.26 \\
\hline
\end{tabular}
\caption{\emph{Loss metrics.} $p=0$ is squared L2 loss. $p=0.5$ is Pseudo-Huber loss.}
\label{tab:loss}
\end{subtable}
\hfill
\begin{subtable}{0.3\linewidth}
\centering
\begin{tabular}{cc}
\hline
$\omega$ & FID, 1-NFE \\
\hline
1.0 (w/o cfg) & 61.06 \\
1.5 & 32.51 \\
2.0 & 19.05 \\
3.0 & \textbf{14.76} \\
5.0 & 20.12 \\
\hline
\end{tabular}
\caption{\emph{CFG scale.} Our method supports 1-NFE CFG sampling.}
\label{tab:cfg}
\end{subtable}

\vspace{-1.0em}

% \begin{subtable}{0.3\linewidth}
% \centering
% \begin{tabular}{cc}
% \hline
% \% of $r \ne t$ & FID, 1-NFE \\
% \hline
% 0\% (=FM) & 328.91 \\
% 25\% & \textbf{61.06} \\
% 50\% & 63.14 \\
% 100\% & 67.32 \\
% \hline
% \end{tabular}
% \caption{Ratio of sampling $r \ne t$. The $0\%$ entry reduces to the standard Flow Matching baseline.}
% \end{subtable}
% \hfill
% %
% \begin{subtable}{0.3\linewidth}
% \centering
% \begin{tabular}{cc}
% \hline
% jvp tangent & FID, 1-NFE \\
% \hline
% $(v,0,1)$ & \textbf{61.06} \\
% $(v,0,0)$ & 268.06 \\
% $(v,1,0)$ & 329.22 \\
% $(v,1,1)$ & 137.96 \\
% \hline
% \end{tabular}
% \caption{JVP computation. The correct jvp tangent is $(v,0,1)$ for Jacobian $(\partial_z u, \partial_r u, \partial_t u)$.}
% \end{subtable}
% \hfill
%
\begin{subtable}{0.4\linewidth}
\centering
\begin{tabular}{cc}
\hline
$t,r$ sampler & FID, 1-NFE \\
\hline
uniform$(0,1)$ & 65.73 \\
lognorm$(-0.2,1.0)$ & 63.56 \\
lognorm$(-0.2,1.2)$ & 62.27 \\
lognorm$(-0.4,1.0)$ & \textbf{59.87} \\
lognorm$(-0.4,1.2)$ & 59.94 \\
\hline
\end{tabular}
\caption{\emph{Time samplers.} $t$ and $r$ are sampled from the specific sampler.}
\label{tab:time}
\end{subtable}
\hfill
\begin{subfigure}{0.55\linewidth}
\centering
\includegraphics[width=\linewidth]{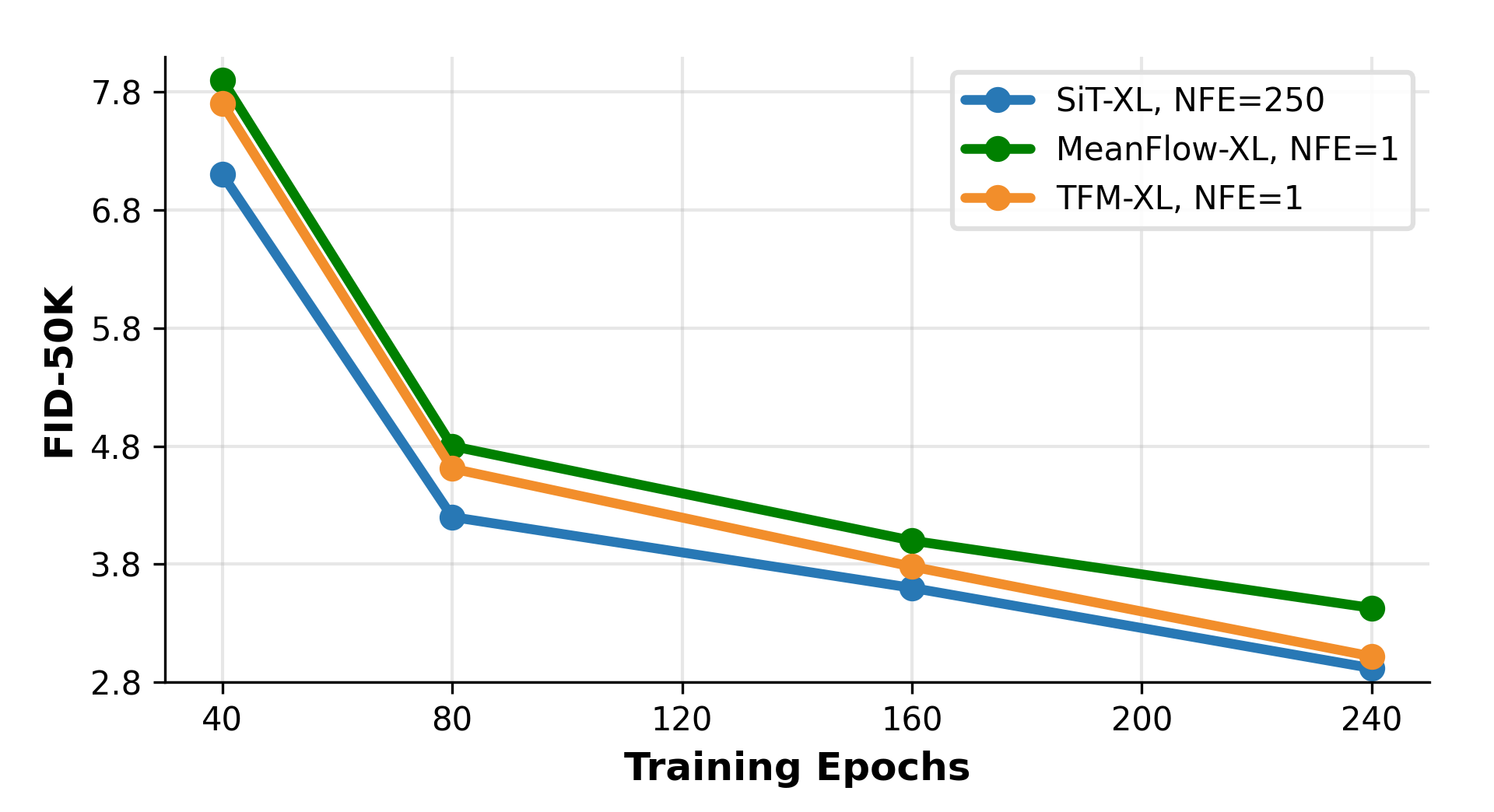}
\caption{\emph{Comparison of FID-50K Score over Training Iterations on ImageNet $256 \times 256$ Dataset.}}
\label{fig:train}
\end{subfigure}

\vspace{-1.0em}

\caption{\emph{Ablation study on 1-NFE ImageNet $256\times256$ generation.} FID-50K is evaluated. Default configurations are: B/4 backbone, 80-epoch training from scratch.}

\vspace{-2.0em}

\end{table*}

\paragraph{Sampling Time Steps $(t,r)$.}
The choice of the sampling distribution for time steps is known to have a significant impact on generation quality. In our framework, we sample $(t,r)$ using logit-normal distributions, consistent with the implementation described in Sec.~\ref{sec:implementation}. Specifically, we consider two logit-normal distributions: one for sampling the base time $t$ and another for sampling the relative offset that determines $r$. This design ensures $0\leq t\leq r\leq 1$ by construction while allowing flexible control over the density of sampled time pairs. We evaluate different hyperparameter settings of these logit-normal samplers in Tab.~\ref{tab:time}. The results show that logit-normal sampling consistently outperforms alternative choices, aligning with observations reported in prior flow-matching-based methods.

\paragraph{Loss Metrics.}
As discussed in Sec.~\ref{sec:implementation}, our training objective adopts a Bregman divergence instantiated via adaptively weighted squared $\ell_2$ loss. While the overall loss formulation remains the same across experiments, we vary the exponent $p$ that controls the adaptive weighting, which effectively changes the loss metric. The corresponding results are summarized in Tab.~\ref{tab:loss}. We find that $p=1$ achieves the best overall performance, indicating a strong benefit from aggressive adaptive weighting. Setting $p=0.5$, which resembles the Pseudo-Huber loss, also yields competitive results. In contrast, the standard squared $\ell_2$ loss ($p=0$) underperforms relative to other choices, though it still leads to meaningful one-step generation. These trends are consistent with prior findings on the importance of loss reweighting for few-step or single-step generative models.

\paragraph{Guidance Scale.}
We further investigate the effect of classifier-free guidance (CFG) within our transition flow framework. The results, reported in Tab.~\ref{tab:cfg}, show that increasing the guidance scale significantly improves generation quality. This behavior is consistent with observations in multi-step diffusion and flow models. Importantly, our CFG formulation, introduced in Sec.~\ref{sec:implementation}, is fully compatible with one-step (1-NFE) sampling and does not introduce additional inference cost beyond a constant factor.

% \paragraph{Scalability.}
% Finally, we study the scalability of our method with respect to model size and training duration. Fig.~\ref{} reports the one-step FID scores across different configurations. Similar to Transformer-based diffusion and flow models, our transition flow models exhibit clear and consistent scaling trends: larger models and longer training lead to systematically improved performance. These results suggest that the proposed method retains favorable scalability properties while enabling efficient one-step generation.

\section{Conclusion} \label{sec:conclusion}

In this work, we introduced \textbf{Transition Flow Matching}, a principled framework for few-step generative modeling. Instead of learning local velocity fields as in conventional Flow/Diffusion models, our method directly models the generation trajectory through transition dynamics, providing a global perspective on generative flows. We derive the \textbf{Transition Flow Identity} and develop a theoretically grounded objective that enables end-to-end training from scratch, while also establishing a unified view that connects our framework with Mean Velocity methods. 
% Extensive experiments demonstrate that modeling transition flows leads to competitive performance and provides an effective approach for efficient generative modeling.

\section{Proof of Transition Identity}
\label{proof:transitionIdentity}

\subsection{Notation and Preliminary}
\paragraph{Random variables and realizations.}
We work in $\mathbb{R}^d$. Uppercase letters (e.g., $X_t, Z$) denote random variables (RVs), and lowercase letters (e.g., $x_t, z$) denote their realizations (points/values). For a density (or probability law) of an RV $X_t$, we write $p_t(\cdot)$, and for a conditional density we write $p_{t\mid Z}(\cdot\mid z)$. Expectations are denoted by $\mathbb{E}[\cdot]$.

\paragraph{Source/target distributions and coupling.}
Let $X_0 \sim p_0$ be the \emph{source} distribution (e.g., standard Gaussian noise) and $X_1 \sim p_1$ be the \emph{target} distribution (e.g., images).
A generative model constructs a continuous path of distributions $\{p_t\}_{t\in[0,1]}$ that transports $p_0$ to $p_1$.
Let $(X_0,X_1)$ be any coupling on $\mathbb{R}^d$ with joint density $\pi$ whose marginals are $p_0$ and $p_1$ (not necessarily independent). In this work, we follow the standard Flow Matching setting, the \emph{source} and the \emph{target} distributions are independent: $\pi(x_0, x_1) = p_0(x_0) \, p_1(x_1)$.

\paragraph{Conditioning variable and conditional paths~\cite{lipman2023flow}.}
We use a conditioning RV $Z$ to index conditional paths.
Conditioned on $Z=z$, we obtain conditional coupling $(X_0^Z, X_1^Z)$ and a conditional path $\{X_t^Z\}_{t\in[0,1]}$ with conditional density $p_{t\mid Z}(\cdot\mid z)$.
The marginal path is $\{X_t\}_{t\in[0,1]}$, with density $p_t(\cdot)$, satisfying:
\begin{equation}
\label{proof:eq:bayes_xt}
p_t(x_t) \;=\; \int p_{t\mid Z}(x_t\mid z)\,p_Z(z) dz,
\qquad
X_t \sim p_t,\ \ X_t\mid(Z=z)\sim p_{t\mid Z}(\cdot\mid z).
\end{equation}

\paragraph{Interpolant.}
We consider a general interpolant between the \emph{marginal} endpoints $X_0$ and $X_1$,
specified by scalar functions $\alpha:[0,1]\to\mathbb{R}$ and $\beta:[0,1]\to\mathbb{R}$:
\begin{equation}
\label{proof:eq:generalInterpolant}
X_t \;=\; \alpha(t)\,X_0 + \beta(t)\,X_1,
\qquad t\in[0,1].
\end{equation}
We assume the boundary conditions $\alpha(0)=1,\ \beta(0)=0$ and $\alpha(1)=0,\ \beta(1)=1$, so that $X_{t=0}=X_0$ and $X_{t=1}=X_1$.
If $\alpha,\beta$ are differentiable, then
\begin{equation}
\label{proof:eq:generalInterpolant_dot}
\frac{d}{dt}X_t
\;=\;
\dot{\alpha}(t)\,X_0 + \dot{\beta}(t)\,X_1.
\end{equation}
% For realizations, we write
% \[
% x_t=\alpha(t)x_0+\beta(t)x_1,
% \qquad
% \dot{x}_t=\dot{\alpha}(t)x_0+\dot{\beta}(t)x_1.
% \]
Conditioned on $Z=z$, the same schedules induce the conditional path
$X_t^Z=\alpha(t)X_0^Z+\beta(t)X_1^Z$ and $\frac{d}{dt}X_t^Z=\dot{\alpha}(t)X_0^Z+\dot{\beta}(t)X_1^Z$.

\paragraph{Flow Matching vector fields~\cite{lipman2023flow}.}
Let $v(x_t,t\mid z)\in\mathbb{R}^d$ denote a \emph{conditional} velocity field that transports the conditional density $p_{t\mid Z}(\cdot\mid z)$ along time.
The corresponding \emph{marginal} velocity field is defined by conditional expectation:
\begin{equation}
\label{proof:eq:marginalV}
v(x_t,t)
\;=\;
\int v(x_t,t\mid z)\,p_{Z\mid t}(z\mid x_t)\,dz
\;=\;
\mathbb{E}\!\left[\,v(X_t,t\mid Z)\,\middle|\,X_t=x_t\,\right].
\end{equation}
A marginal trajectory, i.e., Flow Matching generation trajectory, follows the ODE:
\begin{equation}
\label{proof:eq:ode_marginal}
\frac{d x_t}{dt} = v(x_t,t), \qquad t\in[0,1],
\end{equation}
and similarly a conditional trajectory follows $\frac{d x_t^z}{dt}=v(x_t,t\mid z)$.

\paragraph{Continuity (transport) equations~\cite{lipman2023flow, lipman2024flowmatchingguidecode}.}
The evolution of densities induced by these velocity fields is characterized by the continuity (transport) equation.
For the marginal density $p_t$,
\begin{equation}
\label{proof:eq:continuity_marginal}
\partial_t p_t(x_t) + \nabla\!\cdot\!\big(p_t(x_t)\,v(x_t,t)\big)=0,
\qquad t\in[0,1].
\end{equation}
Conditioned on $Z=z$, the conditional density $p_{t\mid Z}(\cdot\mid z)$ satisfies
\begin{equation}
\label{proof:eq:continuity_conditional}
\partial_t p_{t\mid Z}(x_t\mid z) + \nabla\!\cdot\!\big(p_{t\mid Z}(x_t\mid z)\,v(x_t,t\mid z)\big)=0,
\qquad t\in[0,1].
\end{equation}
Flow Matching learns a parameterized velocity field (or equivalent dynamics) so that the induced marginal path $\{p_t\}$ solves Eq.\ \eqref{proof:eq:continuity_marginal} with boundary conditions $p_{t=0}=p_0$ and $p_{t=1}=p_1$.

\paragraph{Learning Flow Matching.}
Flow Matching introduce a velocity field model $v^\theta(X_t,t)$ to learn $v(X_t,t)$, ideally, by minimizing the marginal Flow Matching loss:
\begin{equation}
\label{proof:eq:MFM}
\mathcal{L}_\mathrm{MFM}(\theta) = \mathbb{E}_{t,\;X_t\sim p_t}\,
D\Big(v(X_t,t),\,v^\theta(X_t,t)\Big)
\end{equation}
However, since the marginal velocity $v(X_t,t)$ in Eq.\eqref{proof:eq:marginalV} is not tractable, so the marginal loss Eq.\eqref{proof:eq:MFM} above cannot be computed as is. Instead, we minimize the conditional Flow Matching loss:
\begin{equation}
\label{proof:eq:CFM}
\mathcal{L}_\mathrm{CFM}(\theta) = \mathbb{E}_{t,\;Z,\;X_t\sim p_{t\mid Z}(\cdot \mid Z)}\,
D\Big(v(X_t,t \mid Z),\,v^\theta(X_t,t)\Big)
\end{equation}
The two losses Eq.\eqref{proof:eq:MFM} and Eq.\eqref{proof:eq:CFM} are equivalent for learning purposes, since their gradients coincide:
\begin{theorem}[Gradient equivalence of Flow Matching~\cite{lipman2024flowmatchingguidecode}]\label{proof:thm:MequivC_FM}
The gradients of the marginal Flow Matching loss and the conditional Flow Matching loss coincide:
\begin{equation}
\label{proof:eq:marginalEqvconditional_FM}
\nabla_\theta \mathcal{L}_\mathrm{MFM}(\theta)
\;=\;
\nabla_\theta \mathcal{L}_\mathrm{CFM}(\theta)
\end{equation}
In particular, the minimizer of the conditional Flow Matching loss is the marginal velocity $v(x_t, t)$.
\end{theorem}

\begin{remark}[Standard Flow Matching] \label{proof:remark:FM}
Flow Matching sets the conditioning variable to be the endpoint pair
\[
Z \;=\; (X_0,X_1),
\]
so that conditioning on $Z=z$ fixes the endpoint pair $(X_0^Z,X_1^Z) = (X_0,X_1)$.
Choosing the linear schedules $\alpha(t)=1-t$ and $\beta(t)=t$ yields the conditional path
\[
X_t^Z \;=\; (1-t)X_0^Z + tX_1^Z \;=\; (1-t)X_0 + tX_1,\qquad 0\le t\le 1
\]
hence the time-derivative of this conditional path is constant:
\[
\frac{d}{dt}X_t^Z \;=\; X_1^Z - X_0^Z \;=\; X_1 - X_0.
\]
Therefore, for $Z=(X_0,X_1)$, the conditional velocity used as supervision in the conditional Flow Matching loss is a constant:
\[
v(X_t,t\mid Z) \;=\; X_1^Z - X_0^Z \;=\; X_1 - X_0,
\]

With this setup, the Flow Matching objective in Eq.~\eqref{proof:eq:CFM} can be written explicitly as
\[
\mathcal{L}_\mathrm{FM}(\theta)
\;=\;
\mathbb{E}_{t,\;X_0\sim p_0,\;X_1\sim p_1}\,
D\Big(X_1 - X_0,\;v^\theta(X_t,t)\Big),
\qquad
X_t=(1-t)X_0+tX_1,
\]
i.e., one samples $t\sim\mathrm{Unif}[0,1]$, draws independent endpoints $X_0\sim p_0$ and $X_1\sim p_1$,
forms the interpolated state $X_t$, and regresses the model $v^\theta(X_t,t)$ to the constant target $X_1-X_0$.
By Theorem~\ref{proof:thm:MequivC_FM}, minimizing this conditional loss yields the marginal velocity field $v(x_t,t)$ in Eq.~\eqref{proof:eq:marginalV}.
\end{remark}

\subsection{Transition Flow Identity} \label{proof:sec:TranFlowIden}
To connect with previous works~\cite{geng2025mean, geng2025improved}, we define the average velocity $u(x_t, t, r)$ as:
\begin{equation}
\label{proof:eq:averageV}
    (r-t)u(x_t, t, r) = x_{t \to r} - x_t = \int_{t}^{r} v(x_\tau, \tau) \, d\tau   \quad \quad 0\leq t \leq r \leq 1
\end{equation}
where $v(x_\tau, \tau)$ is the marginal velocity of Flow Matching in time step $\tau \in [0,1]$ for state $x_\tau$, as shown in Eq.\eqref{proof:eq:marginalV}.

In Eq.\eqref{proof:eq:averageV}, $x_t$ is the current state at time step $t$, $x_{t \to r}$ is the transition state at time step $r$, which comes from $x_t$ at time step $t$. 
\begin{equation}
x_{t \to r} = x_t + \int_{t}^{r} v(x_\tau, \tau)
\end{equation}
Note here, akin to the marginal velocity of Flow Matching $v(x_t, t)$ as Eq.\eqref{proof:eq:marginalV}, $x_{t \to r}$ is the marginal transition state:
% \begin{equation}
% x_t = \int x_t^z \; p_{Z\mid \tau}(z\mid x_t) dz = \mathbb{E}  [X_t^Z \mid X_t = x_t] 
% \end{equation}
\begin{equation}
\label{proof:eq:marginalstate}
x_{t \to r} = \int x_{t \to r}^z \; p_{Z\mid t}(z\mid x_t) dz = \mathbb{E}  [X_{t \to r}^Z \mid X_t = x_t] 
\end{equation}
where $x_{t \to r}^z$ is the conditional transition state: 
\begin{equation}
x_{t \to r}^z = x_t + \int_{t}^{r} v(x_\tau, \tau \mid z) \, d\tau
\end{equation}

The average velocity $u(x_t, t, r)$ is the displacement between two time steps $t$ and $r$ divided by the time interval $r-t$:
\begin{equation}
u(x_t, t, r) \triangleq \frac{1}{r-t} \int_{t}^{r} v(x_\tau, \tau) \, d\tau
\end{equation}

Differentiate both sides of Eq.\eqref{proof:eq:averageV} with respect to t, treating r as independent of t. We have:
\begin{equation}
\label{proof:eq:u_identity}
\frac{d}{dt} (r-t)u(x_t, t, r) = \frac{d}{dt} \int_{t}^{r} v(x_\tau, \tau) \, d\tau \implies u(x_t, t, r) = v(x_t, t) + (r-t) \frac{d}{dt} u(x_t, t, r)
\end{equation}

Our learning target is the Transition Flow $X(x_t, t, r)$ that 
transit state $x_t$ at time step $t$ to $x_{t \to r}$ at time step $r$: $X(x_t, t, r) = x_{t \to r}$, which yield:
\begin{equation}
\label{proof:eq:marginalTransitionFlow}
X(x_t, t, r) = \int X(x_t, t, r \mid z) \; p_{Z\mid t}(z\mid x_t) dz = \mathbb{E}  [X(X_t, t, r \mid Z) \mid X_t = x_t] 
\end{equation}
where $X(x_t, t, r \mid z) = x_{t \to r}^z$ is conditional transition state.
Note that Eq.\eqref{proof:eq:marginalTransitionFlow} is identical to Eq.\eqref{proof:eq:marginalstate}. 

From Eq.\eqref{proof:eq:averageV}, we build the association between the average velocity and the Transition Flow (transition state):
\begin{equation}
\label{proof:eq:u2x}
u(x_t, t, r) = \frac{x_{t \to r} - x_t}{r-t} = \frac{X(x_t, t, r) - x_t}{r-t}
\end{equation}

Substitute Eq.\eqref{proof:eq:u2x} into Eq.\eqref{proof:eq:u_identity}, we obtain the following equation:
\begin{equation}
\label{proof:eq:originalX}
\frac{X(x_t,t,r)-x_t}{r-t}
=
v(x_t,t)
+
(r-t)\,\frac{d}{d t}
\left(
\frac{X(x_t,t,r)-x_t}{r-t}
\right)
\end{equation}
where \(r\) is independent of \(t\), and the time derivative of \(x_t\) is marginal velocity $v(x_t,t)$, given by Eq.\eqref{proof:eq:ode_marginal}.

Define the time derivative of the fraction:
\[
A(t) := \frac{X(x_t,t,r)-x_t}{r-t}.
\]
Using the quotient rule and the fact that \(r\) is constant, we obtain
\[
\frac{d A}{d t}
=
\frac{(r-t)\,\frac{d}{dt}\!\left(X(x_t,t,r)-x_t\right) + \left(X(x_t,t,r)-x_t\right)}{(r-t)^2}.
\]
Since
\[
\frac{d}{dt}\!\left(X(x_t,t,r)-x_t\right)
=
\frac{dX(x_t,t,r)}{dt}
-
v(x_t,t),
\]
we have
\[
\frac{d A}{d t}
=
\frac{(r-t)\left(\frac{dX(x_t,t,r)}{dt}-v(x_t,t)\right)
+
\left(X(x_t,t,r)-x_t\right)}{(r-t)^2}.
\]

The right-hand side of Eq.\eqref{proof:eq:originalX} becomes
\begin{equation}
\begin{aligned}
v(x_t,t)
+
(r-t)\frac{d A}{d t}
&=
v(x_t,t)
+
\frac{(r-t)\left(\frac{dX(x_t,t,r)}{dt}-v(x_t,t)\right)
+
\left(X(x_t,t,r)-x_t\right)}{r-t} \\
&=
v(x_t,t)
+
\frac{dX(x_t,t,r)}{dt}
-
v(x_t,t)
+
\frac{X(x_t,t,r)-x_t}{r-t} \\
&=
\frac{dX(x_t,t,r)}{dt}
+
\frac{X(x_t,t,r)-x_t}{r-t}
\end{aligned}
\end{equation}

Comparing both sides of Eq.\eqref{proof:eq:originalX} yields
\begin{equation}
\begin{aligned}
\frac{X(x_t,t,r)-x_t}{r-t}
&=
\frac{dX(x_t,t,r)}{dt} + \frac{X(x_t,t,r)-x_t}{r-t} \\
X(x_t,t,r) &= X(x_t,t,r) + (r-t) \frac{dX(x_t,t,r)}{dt} \\
\end{aligned}
\end{equation}
Given $X(x_t, t, r) = x_{t \to r}$, the \textbf{Transition Flow Identity} can be shown as:
\begin{equation}
\label{proof:eq:transitionIdentity}
\boxed{
X(x_t,t,r) = x_{t \to r} + (r-t) \frac{d}{dt} X(x_t,t,r)
}
\end{equation}

\subsection{Calculate Time Derivative of Transition Flow}
To compute the $\frac{d}{dt}X(x_t, t, r)$ term, we expand it in terms of partial derivatives:
\begin{equation}
\label{proof:eq:X_timederivative}
\begin{aligned}
\frac{d}{dt} X(x_t,t,r) &= \partial_{x_t}X \cdot \frac{d x_t}{dt} + \partial_{t}X \cdot \frac{d t}{dt} + \partial_{r}X \cdot \frac{d r}{dt} \\
&= \partial_{x_t}X \cdot v(x_t, t) + \partial_{t}X \cdot 1 + \partial_{r}X \cdot 0 \\
&= \partial_{x_t}X \cdot v(x_t, t) + \partial_{t}X
\end{aligned}
\end{equation}
where $\frac{d x_t}{dt} = v(x_t,t)$ as shown in Eq.\eqref{proof:eq:ode_marginal}.
The time derivative shown in Eq.~\eqref{proof:eq:X_timederivative} is given by the Jacobian-vector product (JVP) between the Jacobian matrix of each function ($[\partial_{x_t}X, \partial_{t}X, \partial_{r}X]$) and the corresponding tangent vector ($[v, 0, 1]$). For code implementation, modern libraries such as PyTorch provide efficient JVP calculation interfaces.

\subsection{Towards Tractable Transition Flow Matching Objective} \label{proof:sec:M2C}
Up to this point, the formulations are independent of any network parameterization. We now introduce the Transition Flow model $X^\theta(x_t,t,r)$, parameterized by $\theta$, to learn $X(x_t,t,r)$. 
Formally, we encourage $X^\theta(x_t,t,r)$ to satisfy the \textbf{Transition Flow Identity} as Eq.\eqref{proof:eq:transitionIdentity}.
To this end, ideally, we can minimize the marginal Transition Flow Matching objective (M-TFM):
\begin{equation}
\begin{aligned}
\label{proof:eq:MTFM}
\mathcal{L}_\mathrm{M-TFM}(\theta) = \mathbb{E}_{t,\;r,\;X_t\sim p_t}\,
D\Big(\mathrm{sg}\big[X_{\mathrm{tgt}}^{\mathrm{m}}(X_t,t, r)\big],\,X^\theta(X_t,t, r)\Big)
\end{aligned}
\end{equation}
\begin{equation}
\label{proof:eq:marginalXtgt}
\text{where} \quad \quad X_{\mathrm{tgt}}^{\mathrm{m}}(X_t,t, r) = X_{t \to r} + (r-t) \frac{d}{dt} X^\theta(X_t,t,r)
\end{equation}
where $D(\cdot)$ represents a Bregman divergence (e.g MSE) to regress our learnable Transition Flow $X^\theta(x_t,t,r)$ onto the target $X_{\mathrm{tgt}}^{\mathrm{m}}(x_t,t, r)$, moreover, $\mathrm{sg}[\cdot]$ donates the stop gradient (sg) operation, indicating $X_{\mathrm{tgt}}^{\mathrm{m}}(x_t,t, r)$ serves as ground-truth in this loss function and is independent of optimization.

However, the marginal state $x_{t \to r}$ in Eq.\eqref{proof:eq:marginalXtgt} and the marginal velocity $v(x_t, t)$ used in calculation of $\frac{d}{dt} X^\theta(x_t,t,r)$ Eq.\eqref{proof:eq:X_timederivative} are not tractable, so the marginal loss Eq.\eqref{proof:eq:MTFM} above cannot be computed as is.

% Subtracting the common term on both sides, we conclude that
% \[
% \boxed{
% \frac{dX(x_t,t,r)}{dt} = 0.
% }
% \]

Instead, we minimize conditional Transition Flow Matching loss (C-TFM), which is tractable:
\begin{equation}
\begin{aligned}
\label{proof:eq:CTFM}
\mathcal{L}_\mathrm{C-TFM}(\theta) = \mathbb{E}_{t,\;r,\;Z,\;X_t\sim p_{t\mid Z}(\cdot \mid Z)}\,
D\Big(\mathrm{sg}\big[X_{\mathrm{tgt}}^{\mathrm{c}}(X_t,t, r \mid Z)\big],\,X^\theta(X_t,t, r)\Big)
\end{aligned}
\end{equation}
\begin{equation}
\label{proof:eq:conditionalXtgt}
\text{where} \quad \quad X_{\mathrm{tgt}}^{\mathrm{c}}(X_t,t, r \mid Z) = X_{t \to r}^Z + (r-t) \frac{d}{dt} X^\theta(X_t,t,r)
\end{equation}
where the conditional state $x_{t \to r}^z$ in Eq.\eqref{proof:eq:conditionalXtgt} and the conditional velocity $v(x_t, t \mid z)$ used in calculation of $\frac{d}{dt} X^\theta(x_t,t,r)$ Eq.\eqref{proof:eq:X_timederivative} are tractable (see Eq.\eqref{proof:eq:M2C} for details), gives us a computable loss function Eq.\eqref{proof:eq:CTFM}.

The two losses Eq.\eqref{proof:eq:MTFM} and Eq.\eqref{proof:eq:CTFM} are equivalent for learning purposes, since their gradients coincide:
\begin{theorem}[Gradient equivalence of Transition Flow Matching]\label{proof:thm:MequivC}
The gradients of the marginal Transition Flow Matching loss and the conditional Transition Flow Matching loss coincide:
\begin{equation}
\label{proof:eq:marginalEqvconditional}
\nabla_\theta \mathcal{L}_\mathrm{M-TFM}(\theta)
\;=\;
\nabla_\theta \mathcal{L}_\mathrm{C-TFM}(\theta)
\end{equation}
In particular, the minimizer of the conditional Transition Flow Matching loss is the marginal target $X_{\mathrm{tgt}}^{\mathrm{m}}(x_t,t, r)$, which satisfies the Transition Flow Identity Eq.\eqref{proof:eq:transitionIdentity}.
\end{theorem}
\begin{proof}[Proof of Theorem~\ref{proof:thm:MequivC}]
We show Eq.\eqref{proof:eq:marginalEqvconditional} by a direct computation.
\begin{equation}
\label{proof:eq:ML2CL}
\scriptsize {
\begin{aligned}
\nabla_\theta \mathcal{L}_\mathrm{M\text{-}TFM}(\theta)
&=
\nabla_\theta \mathbb{E}_{t,\;r,\;X_t\sim p_t}\,
D\Big(\mathrm{sg}\big[X_{\mathrm{tgt}}^{\mathrm{m}}(X_t,t,r)\big],\,X^\theta(X_t,t,r)\Big)
\\
&\overset{(\mathrm{a})}{=}
\mathbb{E}_{t,\;r,\;X_t\sim p_t}\;
\nabla_\theta
D\Big(\mathrm{sg}\big[X_{\mathrm{tgt}}^{\mathrm{m}}(X_t,t,r)\big],\,X^\theta(X_t,t,r)\Big)
\\
&\overset{(\mathrm{b})}{=}
\mathbb{E}_{t,\;r,\;X_t\sim p_t}\;
\nabla_{2}D\Big(X_{\mathrm{tgt}}^{\mathrm{m}}(X_t,t,r),\,X^\theta(X_t,t,r)\Big)\,
\nabla_\theta X^\theta(X_t,t,r)
\\
&\overset{(\mathrm{c})}{=}
\mathbb{E}_{t,\;r,\;X_t\sim p_t}\;
\nabla_{2}D\Big(
\mathbb{E}_{Z\sim p_{Z\mid t}(\cdot\mid X_t)}\big[X_{\mathrm{tgt}}^{\mathrm{c}}(X_t,t,r\mid Z)\big],
\,X^\theta(X_t,t,r)\Big)\,
\nabla_\theta X^\theta(X_t,t,r)
\\
&\overset{(\mathrm{d})}{=}
\mathbb{E}_{t,\;r,\;X_t\sim p_t}\;
\mathbb{E}_{Z\sim p_{Z\mid t}(\cdot\mid X_t)}\Big[
\nabla_{2}D\Big(X_{\mathrm{tgt}}^{\mathrm{c}}(X_t,t,r\mid Z),\,X^\theta(X_t,t,r)\Big)\,
\nabla_\theta X^\theta(X_t,t,r)
\Big]
\\
&\overset{(\mathrm{e})}{=}
\mathbb{E}_{t,\;r,\;X_t\sim p_t}\;
\mathbb{E}_{Z\sim p_{Z\mid t}(\cdot\mid X_t)}\Big[
\nabla_{\theta}
D\Big(\mathrm{sg}\big[X_{\mathrm{tgt}}^{\mathrm{c}}(X_t,t,r\mid Z)\big],\,X^\theta(X_t,t,r)\Big)
\Big]
\\
&\overset{(\mathrm{f})}{=}
\nabla_\theta\,
\mathbb{E}_{t,\;r,\;Z,\;X_t\sim p_{t\mid Z}(\cdot\mid Z)}\,
D\Big(\mathrm{sg}\big[X_{\mathrm{tgt}}^{\mathrm{c}}(X_t,t,r\mid Z)\big],\,X^\theta(X_t,t,r)\Big)
\\
&=
\nabla_\theta \mathcal{L}_\mathrm{C\text{-}TFM}(\theta).
\end{aligned} }
\end{equation}

\paragraph{Explanations of labeled steps.}
\begin{enumerate}\itemsep 4pt
\item[(a)] Interchange $\nabla_\theta$ and $\mathbb{E}$ by the Leibniz rule.
\item[(b)] The stop-gradient makes the first argument $\mathrm{sg}[X_{\mathrm{tgt}}^{\mathrm{m}}(X_t,t,r)]$ $\theta$-independent for differentiation, hence the gradient flows only through the second argument. Applying the chain rule yields
$\nabla_\theta D(\mathrm{sg}[\cdot],X^\theta)=\nabla_2 D(\cdot,X^\theta)\,\nabla_\theta X^\theta$.
\item[(c)] Use the definitions Eq.\eqref{proof:eq:marginalXtgt} and Eq.\eqref{proof:eq:conditionalXtgt} together with the marginal--conditional relation for transition states Eq.\eqref{proof:eq:marginalTransitionFlow} and velocity Eq.\eqref{proof:eq:marginalV}:
\begin{equation}
\label{proof:eq:M2C}
\scriptsize {
\begin{aligned}
X_{\mathrm{tgt}}^{\mathrm{m}}(X_t,t,r)
&=
X_{t\to r}+(r-t)\frac{d}{dt}X^\theta(X_t,t,r) \\
&=
X(X_t, t, r)+(r-t)\frac{d}{dt}X^\theta(X_t,t,r) \\
&=
X(X_t, t, r)+(r-t)\Big(\partial_{X_t}X^\theta(X_t,t,r)\cdot v(X_t,t)
+\partial_t X^\theta(X_t,t,r)\Big) \\
&=
\mathbb{E}_{Z\sim p_{Z\mid t}(\cdot\mid X_t)}\!\Big[X(X_t, t, r \mid Z)+(r-t)\Big(\partial_{X_t}X^\theta(X_t,t,r)\cdot v(X_t,t \mid Z) \\
&+\partial_t X^\theta(X_t,t,r)\Big)\Big] \\
&=
\mathbb{E}_{Z\sim p_{Z\mid t}(\cdot\mid X_t)}\!\Big[X(X_t, t, r \mid Z)+(r-t)\frac{d}{dt}X^\theta(X_t,t,r)\Big] \\
&=
\mathbb{E}_{Z\sim p_{Z\mid t}(\cdot\mid X_t)}\!\Big[X_{t\to r}^Z+(r-t)\frac{d}{dt}X^\theta(X_t,t,r)\Big] \\
&=
\mathbb{E}_{Z\sim p_{Z\mid t}(\cdot\mid X_t)}\!\big[X_{\mathrm{tgt}}^{\mathrm{c}}(X_t,t,r\mid Z)\big],
\end{aligned} }
\end{equation}
where $\partial_{X_t}X^\theta(X_t,t,r)$ and $\partial_t X^\theta(X_t,t,r)$ are $Z$-independent given $(X_t,t,r)$.
\item[(d)] Since $D(\cdot,\cdot)$ is a Bregman divergence, its gradient with respect to the second argument, $\nabla_2 D(a,b)$, is affine in the first argument $a$ for fixed $b$. Conditioning on $X_t$, this implies the (conditional) expectation can be moved inside:
\[
\nabla_2 D\Big(\mathbb{E}[A\mid X_t],\,b\Big)
=
\mathbb{E}\big[\nabla_2 D(A,b)\mid X_t\big].
\]
Apply this with $A=X_{\mathrm{tgt}}^{\mathrm{c}}(X_t,t,r\mid Z)$ and $b=X^\theta(X_t,t,r)$.
\item[(e)] Reverse the chain rule as in (b): because $\mathrm{sg}[X_{\mathrm{tgt}}^{\mathrm{c}}]$ freezes the first argument, we have
$\nabla_{2}D(\cdot,X^\theta)\,\nabla_\theta X^\theta=\nabla_\theta D(\mathrm{sg}[\cdot],X^\theta)$.
\item[(f)] Use Bayes' rule to swap the sampling orders:
\[
\mathbb{E}_{X_t\sim p_t}\,\mathbb{E}_{Z\sim p_{Z\mid t}(\cdot\mid X_t)}[\cdot]
=
\mathbb{E}_{Z}\,\mathbb{E}_{X_t\sim p_{t\mid Z}(\cdot\mid Z)}[\cdot],
\]
and interchange $\nabla_\theta$ with $\mathbb{E}$ (as in (a)) to recognize $\nabla_\theta \mathcal{L}_\mathrm{C\text{-}TFM}(\theta)$ in Eq.\eqref{proof:eq:CTFM}.
\end{enumerate}

Therefore,
$\nabla_\theta \mathcal{L}_\mathrm{M\text{-}TFM}(\theta)=\nabla_\theta \mathcal{L}_\mathrm{C\text{-}TFM}(\theta)$,
which proves Eq.\eqref{proof:eq:marginalEqvconditional}. \qed

% \medskip
\begin{remark}[Minimizer]
Since the two objectives Eq.\eqref{proof:eq:MTFM} and Eq.\eqref{proof:eq:CTFM} have identical gradients, they share the same stationary points. Moreover, because
\[
X_{\mathrm{tgt}}^{\mathrm{m}}(X_t,t,r)
=
\mathbb{E}_{Z\sim p_{Z\mid t}(\cdot\mid X_t)}\big[X_{\mathrm{tgt}}^{\mathrm{c}}(X_t,t,r\mid Z)\big],
\]
the conditional objective Eq.\eqref{proof:eq:CTFM} regresses $X^\theta(X_t,t,r)$ toward the marginal target $X_{\mathrm{tgt}}^{\mathrm{m}}(X_t,t,r)$, and $X_{\mathrm{tgt}}^{\mathrm{m}}$ is defined to satisfy the \textbf{Transition Flow Identity} Eq.\eqref{proof:eq:transitionIdentity}.
\end{remark}
\end{proof}

\subsection{Training Transition Flow Matching Model}
Building on the previous development, we now have all the ingredients required to train a Transition Flow Matching model. 
We inherit the setting in Remark~\ref{proof:remark:FM} by taking $Z=(X_0,X_1)$ and using the linear interpolant $X_t = (1-t)X_0+tX_1$, which yields a concrete, standard form of Transition Flow Matching.
\begin{remark}[Standard Transition Flow Matching]
\label{proof:remark:TFM}
We adopt the standard Flow Matching setting in Remark~\ref{proof:remark:FM} by setting
$Z=(X_0,X_1)$, and use the linear interpolant
\[
X_t=(1-t)X_0+tX_1,\qquad 0\le t\le 1.
\]
For any $0\le t\le r\le 1$, the conditional transition state is
\[
X_{t\to r}^Z=(1-r)X_0+rX_1,
\]
and the conditional velocity is constant:
\[
v(X_t,t\mid Z)=X_1-X_0.
\]
The tractable Transition Flow Matching objective in Eq.\eqref{proof:eq:CTFM} can be written explicitly as
\begin{equation}
\begin{aligned}
\label{proof:eq:STFM}
\mathcal{L}_\mathrm{TFM}(\theta) = \mathbb{E}_{t,\;r,\;Z,\;X_t\sim p_{t\mid Z}(\cdot \mid Z)}\,
D\Big(\mathrm{sg}\big[X_{t \to r}^Z + (r-t) \frac{d}{dt} X^\theta(X_t,t,r)\big],\,X^\theta(X_t,t, r)\Big)
\end{aligned}
\end{equation}
\begin{equation}
\label{proof:eq:condTimeDX}
\text{where} \quad \quad \frac{d}{dt} X^\theta(X_t,t,r) = \partial_{x_t}X^\theta(X_t,t,r) \cdot v(X_t, t \mid Z) + \partial_{t}X^\theta(X_t,t,r) \\
\end{equation}

During inference, we recursively apply the model to transform the generation trajectory from the source distribution to the target distribution:
\begin{equation}
\hat{x_r} = X^\theta (x_t, t, r)
\end{equation}
For clarity, we summarize the conceptual training and inference procedures in Algorithm~\ref{alg:tfm-training} and Algorithm~\ref{alg:tfm-multistep-sampling}. In particular, by extending Algorithm~\ref{alg:tfm-multistep-sampling}, the one-step generation procedure is described in Algorithm~\ref{alg:tfm-1step-sampling}.
\end{remark}

% \begin{algorithm}[t]
% \caption{Transition Flow Matching (TFM): Training}
% \label{alg:tfm-training}
% \begin{algorithmic}
% \STATE \textbf{Note:} in PyTorch and JAX, \texttt{jvp} returns the function output and JVP.
% \STATE \textbf{//} \texttt{tfm(x\_t, t, r)}: the function using model $X^\theta(x_t,t,r)$ predicting the transition state at time $r$
% \STATE \textbf{//} \texttt{metric(.)}: e.g., MSE, any Bregman divergence on residuals
% \STATE
% \STATE $x_1 \leftarrow \texttt{sample\_from\_training\_batch}()$ \hfill \textbf{//} sample from target distribution $p_1$
% \STATE $x_0 \leftarrow \texttt{randn\_like}(x_1)$ \hfill \textbf{//} sample from source distribution $p_0$ (Gaussian), independent of $p_1$
% \STATE $(t,r) \leftarrow \texttt{sample\_t\_r}()$ \hfill \textbf{//} $0\le t \le r \le 1$
% \STATE
% \STATE $x_t \leftarrow (1-t)\,x_0 + t\,x_1$ \hfill \textbf{//} current state
% \STATE $v \leftarrow x_1 - x_0$ \hfill \textbf{//} conditional velocity with $Z=(X_0,X_1)$
% \STATE $x_{t\to r} \leftarrow (1-r)\,x_0 + r\,x_1$ \hfill \textbf{//} conditional transition state with $Z=(X_0,X_1)$
% \STATE
% \STATE $X,\ \frac{dX}{dt} \leftarrow \texttt{jvp}(\texttt{tfm}, (x_t,t,r), (v,1,0))$ \hfill \textbf{//} $\frac{d}{dt}X^\theta = \partial_{x_t}X^\theta\cdot v + \partial_t X^\theta$
% \STATE
% \STATE $X_{\text{tgt}} \leftarrow x_{t\to r} + (r-t)\,\frac{dX}{dt}$
% \STATE $\text{error} \leftarrow X - \texttt{stopgrad}(X_{\text{tgt}})$
% \STATE $\text{loss} \leftarrow \texttt{metric}(\text{error})$
% \STATE \textbf{return} loss
% \end{algorithmic}
% \end{algorithm}

\begin{algorithm}[t]
\caption{Transition Flow Matching (TFM): Training}
\label{alg:tfm-training}

% white editor-style box (no tcolorbox)
\setlength{\fboxsep}{8pt}
\noindent
\fcolorbox{CodeFrame}{CodeBG}{
\begin{minipage}{0.933\linewidth}
\ttfamily\small

\cm{Note: In PyTorch and JAX, jvp returns (function\_output, JVP).}\\
\cm{// tfm(x\_t, t, r): uses model $X^\theta(x_t,t,r)$ predicting the transition state at time r}\\
\cm{// metric(.): e.g., MSE, any Bregman divergence on residuals}\\[4pt]

x\_1 = \fn{sample\_from\_training\_batch}() \hfill \cm{// sample from target distribution $p_1$}\\
x\_0 = \fn{randn\_like}(x\_1) \hfill \cm{// sample from source distribution $p_0$ (Gaussian), independent of $p_1$}\\
(t, r) = \fn{sample\_t\_r}() \hfill \cm{// 0 <= t <= r <= 1}\\[4pt]

x\_t = (1 - t)\,x\_0 + t\,x\_1 \hfill \cm{// current state}\\
v = x\_1 - x\_0 \hfill \cm{// conditional velocity with Z = (X\_0, X\_1)}\\
x\_{t\_to\_r} = (1 - r)\,x\_0 + r\,x\_1 \hfill \cm{// conditional transition state with Z = (X\_0, X\_1)}\\[4pt]

(X, dX\_dt) = \fn{jvp}(\model{tfm}, (x\_t, t, r), (v, 1, 0)) \hfill \cm{// d/dt $X^\theta$ = partial\_x $X^\theta$ * v + partial\_t $X^\theta$}\\[4pt]

X\_tgt = x\_{t\_to\_r} + (r - t)\,dX\_dt\\
error = X - \fn{stopgrad}(X\_tgt)\\
loss = \fn{metric}(error)\\[2pt]

\textbf{return} loss

\end{minipage}
}

\end{algorithm}
% \begin{figure}[t]
% \centering
% \begin{minipage}[t]{0.60\linewidth}
% \centering
% \begin{algorithm}[H]
% \caption{TFM: Multi-step Sampling with Arbitrary Step Sizes}
% \label{alg:tfm-multistep-sampling}
% \begin{algorithmic}
% \STATE \textbf{//} Choose any time grid: $0=t_0 < t_1 < \cdots < t_K=1$
% \STATE $x \leftarrow \texttt{randn}(x\_shape)$ \hfill \textbf{//} $x=x_{t_0}=x_0$
% \FOR{$k=0$ \textbf{to} $K-1$}
%     \STATE $t \leftarrow t_k,\ \ r \leftarrow t_{k+1}$
%     \STATE $x \leftarrow \texttt{tfm}(x, t, r)$ \hfill \textbf{//} $x = x_{t\to r}$
% \ENDFOR
% \STATE \textbf{return} $x$ \hfill \textbf{//} $x = x_{t_K}=x_1$
% \end{algorithmic}
% \end{algorithm}
% \end{minipage}\hfill
% \begin{minipage}[t]{0.35\linewidth}
% \centering
% \begin{algorithm}[H]
% \caption{TFM: 1-step Sampling}
% \label{alg:tfm-1step-sampling}
% \begin{algorithmic}
% \STATE $x \leftarrow \texttt{randn}(x\_shape)$ \hfill \textbf{//} $x = x_0$
% \STATE $x \leftarrow \texttt{tfm}(x, 0, 1)$ \hfill \textbf{//} $x = x_1$
% \STATE \textbf{return} $x$
% \end{algorithmic}
% \end{algorithm}
% \end{minipage}
% \end{figure}

\begin{figure}[t]
\centering

\begin{minipage}[t]{0.60\linewidth}
\centering
\begin{algorithm}[H]
\caption{TFM: Multi-step Sampling with Arbitrary Step Sizes}
\label{alg:tfm-multistep-sampling}

% white editor-style box (same format as before)
\setlength{\fboxsep}{8pt}
\noindent
\fcolorbox{CodeFrame}{CodeBG}{
\begin{minipage}{0.89\linewidth}
\ttfamily\small

\cm{// Choose any time grid: $0=t_0 < t_1 < \cdots < t_K=1$}\\[4pt]

x = \fn{randn}(x\_shape) \hfill \cm{// x = x\_{t\_0} = x\_0}\\[2pt]

for k = 0 to K - 1:\\
\hspace*{1.5em} t = t\_k,\ \ r = t\_{k+1}\\
\hspace*{1.5em} x = \fn{tfm}(x, t, r) \hfill \cm{// x = x\_{t\_to\_r}}\\[2pt]

\textbf{return} x \hfill \cm{// x = x\_{t\_K} = x\_1}

\end{minipage}
}

\end{algorithm}
\end{minipage}\hfill

\begin{minipage}[t]{0.35\linewidth}
\centering
\begin{algorithm}[H]
\caption{TFM: 1-step Sampling}
\label{alg:tfm-1step-sampling}

% white editor-style box (same format as before)
\setlength{\fboxsep}{8pt}
\noindent
\fcolorbox{CodeFrame}{CodeBG}{
\begin{minipage}{0.81\linewidth}
\ttfamily\small

x = \fn{randn}(x\_shape) \hfill \cm{// x = x\_0}\\
x = \fn{tfm}(x, 0, 1) \hfill \cm{// x = x\_1}\\[2pt]

\textbf{return} x

\end{minipage}
}

\end{algorithm}
\end{minipage}

\end{figure}

% \section*{Acknowledgements}
% Please insert your acknowledgments here.


\bibliography{main}
\bibliographystyle{splncs04}

\end{document}